\documentclass[final,3p,times,twocolumn]{elsarticle}

\usepackage{pdfpages}
\usepackage{xcolor}
\usepackage{hyperref}
\usepackage{multirow}
\usepackage{amsmath}
\usepackage{amsthm}
\usepackage{amssymb}
\usepackage{bm}
\usepackage{graphicx} 
\usepackage{float}
\usepackage{subfig}
\usepackage{booktabs}
\usepackage[ruled]{algorithm2e}
\usepackage{balance}
\usepackage[font=small,skip=2.0pt]{caption}
\usepackage{cite}
\usepackage{diagbox}

\usepackage{svg}
\usepackage{pdfpages}
\usepackage{graphicx}
\newtheorem{theorem}{Theorem}
\newtheorem{corollary}{Corollary}
\newtheorem{ass}{\textbf{Assumption}}
\newtheorem{lemma}{Lemma}
\newtheorem{rem}{Remark}

\theoremstyle{definition}

\newcommand{\mathbold}[1]{\boldsymbol{#1}}

\newtheorem{definition}{Definition}

\usepackage{array}
\newcolumntype{C}[1]{>{\centering\arraybackslash}p{#1}}
\usepackage{mathtools}

\journal{Robotics and Autonomous Systems}

\begin{document}

\begin{frontmatter}



\title{A Safety-Critical Framework for UGVs in Complex Environments:\\ A Data-Driven Discrepancy-Aware Approach}

\author[label1]{Skylar X. Wei}

\affiliation[label1]{organization={Division of Engineering and Applied Sciences,California Institute of Technology},
            addressline={1200 E. California Blvd.}, 
            city={Pasadena},
            postcode={91125}, 
            state={CA},
            country={USA}}

\author[label2]{Lu Gan}

\affiliation[label2]{organization={The School of Aerospace Engineering, Georgia
Institute of Technology},
            addressline={North Avenue}, 
            city={Atlanta},
            postcode={30332}, 
            state={GA},
            country={USA}}

\author[label1]{Joel W. Burdick}


\begin{abstract}

This work presents a novel data-driven multi-layered planning and control framework for the safe navigation of a class of unmanned ground vehicles (UGVs) in the presence of unknown stationary obstacles and additive modeling uncertainties. The foundation of this framework is a novel robust model predictive planner, designed to generate optimal collision-free trajectories given an occupancy grid map, and a paired ancillary controller, augmented to provide robustness against model uncertainties extracted from learning data.  

To tackle modeling discrepancies, we identify both matched (input discrepancies) and unmatched model residuals between the true and the nominal reduced-order models using closed-loop tracking errors as training data. Utilizing conformal prediction, we extract probabilistic upper bounds for the unknown model residuals, which serve to construct a robustifying ancillary controller. Further, we also determine maximum tracking discrepancies, also known as the robust control invariance tube, under the augmented policy, formulating them as collision buffers. Employing a LiDAR-based occupancy map to characterize the environment, we construct a discrepancy-aware cost map that incorporates these collision buffers. This map is then integrated into a sampling-based model predictive path planner that generates optimal and safe trajectories that can be robustly tracked by the augmented ancillary controller in the presence of model mismatches.

The effectiveness of the framework is experimentally validated for autonomous high-speed trajectory tracking in a cluttered environment with four different vehicle-terrain configurations. We also showcase the framework's versatility by reformulating it as a driver-assist program, providing collision avoidance corrections based on user joystick commands.

\end{abstract}




\end{frontmatter}



\section{Introduction} \label{sec:introduction}

\begin{figure*}[h!]
\centering
    \includegraphics[width=0.99\linewidth]{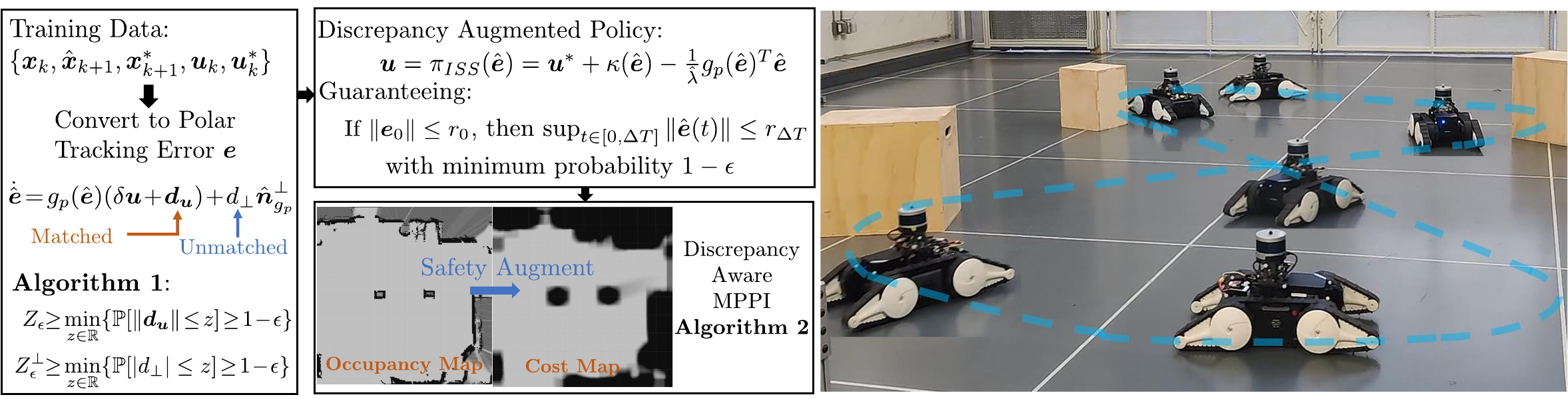}
    \caption{(Top) An overview of our three-step procedure, summarizes the safety-critical framework: (1) offline data-driven model discrepancy identification and learning, (2) augmentation of the control policy based on the learned upper bounds and the associated collision buffers, and (3) a discrepancy-aware MPPI algorithm provides receding-horizon safe trajectory and input pairs.
    (Bottom) snapshots of the UGV movement in high-speed trajectory tracking in a cluttered environment, given unsafe desired trajectories.}
    \label{fig:first_page_summary}
\end{figure*}

Safety is a critical requirement for autonomous exploration of complex environments by unmanned ground vehicles (UGVs) which have a wide range of applications such as agriculture \citep{bonadies2016survey}, search and rescue \citep{9197082}, package delivery \citep{wu2021autonomous}, and mining \citep{9172811}. Specifically, providing collision avoidance guarantees when navigating in unknown and unstructured environments poses unique challenges, such as model mismatches induced by the complex environments \citep{akella2022sample,saveriano2017data} and incomplete obstacle descriptions \citep{dixit2023step,bouman2020autonomous} in the case of \textit{a priori} unknown environment maps. 
Systematic and provable probabilistic frameworks for UGV traversal are needed to provide safety guarantees in such complex scenarios.

Despite the increasing complexity of UGVs and their deployment in diverse terrains, reduced-order models are commonly employed for vehicle planning and control design~\citep{9813568, williams2018robust, mppi_original}. However, a planning and control framework that relies on these models must be robustified against disturbances (e.g. to account for the residual difference between the reduced-order and true models) to guarantee safe and reliable traversal.  This paper introduces a data-driven technique to provide these guarantees even in the presence of a simplified (still underactuated) vehicle model with matched and unmatched model uncertainties \citep{sinha2022adaptive,1323177, pravitra2020}.  





Our method augments the vehicle's controller using learned dynamic discrepancies and applies rigorous uncertainty quantification to provide provably safe operation. Example model discrepancies include wheel slipping and skidding, ignored actuation dynamics, and communication delays. 

The proposed framework automatically learns an upper bound on model residuals from data and systematically calculates the corresponding collision buffers needed to provide guaranteed probabilistic safe navigation of UGV systems in unknown and cluttered environments.  This feature largely eliminates the ``hand tuning" of the underlying planner and controller that is normally required for simplified models. 

Our framework consists of the following contributions (see Fig. ~\ref{fig:first_page_summary}):
\begin{enumerate}
    \item \textbf{Data-Driven Discrepancy Identification}:  Using conformal prediction~\citep{CP_planning_Lindemann}, we extract a probabilistic upper bound of matched (controllable input) discrepancies and unmatched discrepancies from training data, without any assumptions on the discrepancy distributions.
    \item \textbf{Controller Augmentation}: The identified upper bounds are then used to augment the vehicle's nominal ancillary controller to ensure closed-loop stabilizability and robustness against uncertainties.
    \item \textbf{Collision Buffer Construction}: To ensure safety given model uncertainties, we deduce a maximum trajectory tracking deviation of the closed-loop system under the augmented controller. This {\em collision buffer} delineates the boundaries within which the system can operate safely, despite mismatches between the nominal and true vehicle models.
    
    \item \textbf{Discrepancy-Aware Planner}:
    A \textit{discrepancy-aware cost map} is constructed from the identified collision buffer and a sensory-derived occupancy map. This cost map can be seamlessly used with Model Predictive Path Integral (MPPI) to generate optimal finite-horizon trajectories that provably adhere to a user-chosen risk tolerance.
\end{enumerate}

The effectiveness of the proposed framework is demonstrated via hardware experiments on a differential-driven ground vehicle under four different vehicle and terrain configurations. Additionally, we adapted the proposed framework to provide collision avoidance for a human-in-the-loop driver assistance application. 
\subsection{Related Work}

Prior research has studied the impact of unmodeled disturbances on vehicle planning and control using frameworks such as input-to-state safe control~\citep{alan2021safe,NoelCS_multirate} and risk-aware control ~\citep{majumdar2020should,dixit2021risk}. However, these methods often assume \textit{a priori} model knowledge or require a minimum understanding of the disturbances' magnitudes or distributions.  Bayesian Optimization and Reinforcement Learning methods can bypass the uncertainty or model identification steps, directly learning risk-aware control policies in a model-free fashion~\citep{cakmak2020bayesian,makarova2021risk}.  These approaches still come with the assumption of \textit{a priori} knowledge of the disturbances, or that they can be sampled. More recently, the union of Neural Networks with adaptive control~\citep{o2022neural} has demonstrated remarkable tracking improvements in drone control given unknown residual dynamics. However, the theoretical robustness guarantees in~\citep{o2022neural} rely on knowledge of wind disturbance upper bounds, and they do not consider obstacles. 


A learned control policy can be combined with an optimal path planner, as demonstrated in robust model predictive control~\citep{bemporad2007robust, lopez2019dynamic} and chance-constrained stochastic optimal control~\citep{nakka2022trajectory,8613928}. Although these methods construct a deterministic problem surrogate to the original probabilistic one, they often require constraint convexification, such as sequential constraint linearization around fixed points, and optimality can only be reached in infinite sequential iterations. For instance, Monte Carlo Sample propagation~\citep{5477242,5970128} and scenario-based approaches~\citep{calafiore2013stochastic} can require a very large number of samples to reach the desired obstacle avoidance probability guarantees. 

The sampling-based Model Predictive Path Integral control (MPPI) method has proven versatility in off-road racing applications~\citep{mppi_original}. As a model-based strategy, MPPI  may exhibit suboptimal performance in the face of modeling uncertainties and disturbances. Tube-based~\citep{williams2018robust} and robust MPPI variants~\citep{9349120} have been proposed to address these challenges. Both of these methods integrate an ancillary controller, such as iterative Linear Quadratic Gaussian (iLQG), to improve tracking performance with measurement feedback and robustness to uncertainties.  These methods' assumptions on sequential linearization and the additive process and measurement noise being Gaussian limits their applicability to nonlinear systems with generalized additive model uncertainties. 

In unknown environments, UGV navigation often relies on sensor-based occupancy grid maps~\citep{occu_map_1,occu_map_2}, which are commonly used for global path planning~\citep{tsardoulias2016review}.  Grid maps can be used for local optimization-based path planning by converting the maps into signed distance functions~\citep{Oleynikova2016SignedDF, camps2022learning},  which can be locally convexified to serve as the collision-avoiding state constraints. Alternatively, occupancy-based risk maps have been used for sampling-based local planners~\citep{Fan_David_costmap, cai2023evora, Lakshay,dixit2023step}.  To alleviate the computation and convexity burden, the MPPI algorithm allows direct grid-map-based assessment for trajectory costs~\citep{Lakshay}. However, systematic parameter and cost tuning is required to avoid the algorithm making undesirable
decisions~\citep{yin2023shield, reward_hacking}. Lastly, direct estimation of traversability from sensors is also a popular technique for robotic navigation in complex environments leveraging expert heuristics~\citep{dixit2023step, 9813568} or self-supervised learning~\citep{frey2023fast}. However, it is not straightforward to extend these hardware successes to provide provable safety guarantees.

\subsection{Notation and Organization}

The set of integers, positive integers, natural numbers, real numbers, positive reals, and non-negative reals are denoted as $\mathbb{Z}, \mathbb{Z}_{>0}$, $\mathbb{N}$, $\mathbb{R}$, $\mathbb{R}_{>0}$, and $\mathbb{R}_{\geq 0}$, respectively. A sequence of consecutive integers, such as $\{i,\cdots,i+k\}$, is denoted as $\mathbb{Z}_{i}^{i+k}$. A finite sequence of indexed scalars, $\{a_1,\cdots,a_k\}$, is represented as $\{a_i\}_{i=1}^{k}$ and for vectors of $\mathbf{a}$ as $\{\mathbf{a}_i\}_{i=1}^{k}$. Indicator functions are defined as $\mathbf{1}_{b(a)}:\mathbb{R} \to \{0,1\}$, where $b(a)$ is a Boolean proposition over argument $a$. If $b(a)$ is true, the indicator function returns $1$. Else it returns $0$. 

Section \ref{sec:problem_Statement} describes the UGV traversal problem, while Section \ref{sec:prelim} provides the mathematical preliminaries. The characterization of tracking discrepancies using conformal prediction is discussed in Section \ref{sec:dd_discrepancy_identification}, while Section \ref{sec:maximum_discpreancy_tube} augments the nominal ancillary controller to account for the model discrepancies, and derives theoretical bounds on the maximum tracking deviation due to the discrepancies. Section \ref{sec:dA_MPPI} outlines the construction of a discrepancy-aware cost map and the discrepancy-aware MPPI planner. Experimental results, discussions, and future works can be found in Sections~\ref{sec:experiment} and \ref{sec:conclusion}. 

\section{Problem Statement} \label{sec:problem_Statement}
This work investigates the traversal of autonomous ground vehicles, such as differential drive robots, tracked vehicles, and skid steer vehicles. We adopt a widely used model of such vehicles for controller design and motion planning. This kinematic model, e.g. \citep{klancar2017wheeled,388294}, is described as:
\begin{align} \label{eq:origin}
\dot{\mathbold{x}} =   \begin{bmatrix}
    \cos(\theta) & \sin(\theta) & 0 \\
    0 & 0 & 1 
\end{bmatrix}^{T}\begin{bmatrix}
    v\\
    \omega
\end{bmatrix} = g(\mathbold{x})\mathbold{u} = f(\mathbold{x},\mathbold{u}),
\end{align}
where the system state \mbox{$\mathbold{x}=[x,y,\theta]^{T} \in \mathcal{D}^{\mathbold{x}} \subseteq \mathbb{R}^{3}$} consists of the robot's inertial $x,y$ position and its heading angle $\theta$.  The control inputs \mbox{$\mathbold{u} = [v,\omega]^{T} \in \mathcal{D}^{\mathbold{u}} \subseteq \mathbb{R}^{2}$} are the vehicle's linear and angular velocity in the body $x$- and $z$-axes, respectively. The allowable sets of states and inputs are denoted by $\mathcal{D}^{\mathbold{x}}$ and  $\mathcal{D}^{\mathbold{u}}$. 

The nominal model \eqref{eq:origin} is a reduced-order description of a general differential-driven ground vehicle: it assumes that the wheels or tracks contact the ground at all times, ignoring phenomena such as slipping, skidding, motor dynamics, and communication delays. A more comprehensive model, accounting for these neglected factors, can be formulated as:
   \begin{equation}\label{eq:true}
    \dot{\mathbold{x}} = \hat{f}(X) = f(\mathbold{x},\mathbold{u}) + \Tilde{f}(\mathbold{x}, \mathbold{h},\mathbold{u}, \mathbold{d}(t)),
    \end{equation}
where the true dynamics $\hat{f}:Q\to \mathbb{R}^3$ maps state space $Q$ to the vehicle's true linear and angular velocities,  $f(\mathbold{x},\mathbold{u})$ is the nominal model \eqref{eq:origin}, and $\Tilde{f}:Q \to \mathbb{R}^3$ is an \textit{a priori} unknown additive disturbance, that models the discrepancies--it is possibly a function of state $\mathbold{x}$, input $\mathbold{u}$, process noise $\mathbold{d}(t)$, and hidden states $\mathbold{h}$. The variable $X\in Q$ in equation \eqref{eq:true} can be interpreted as a vehicle operating state that includes inputs that correspond to the vehicle velocity $\dot{\mathbold{x}}$ via function $\hat{f}$. To note, we define projections from the vehicle operating state space $\mathcal{Q}$ to the nominal state space $\mathcal{D}^{\mathbold{x}}$ and nominal system input space $\mathcal{D}^{\mathbold{u}}$ for analyzing the discrepancy between the not fully observable true dynamics \eqref{eq:true} and the nominal model\eqref{eq:origin}. 
\begin{definition}
    Let $\mathbold{x}_{X}$ be the \textit{projection of sample $X$ onto $\mathcal{D}^{\mathbold{x}}$}, denoted as $P_{X\to \mathbold{x}}$, is defined as 
\begin{equation}\label{eq:projection_on_state_space}
P_{X\to\mathbold{x}}(X) \triangleq \mathbold{x}_{X} \in \mathcal{D}^{\mathbold{x}}.
\end{equation}
where $\mathcal{D}^{\mathbold{x}}$ is the nominal (reduced-order) admissible states and $X \in Q$ is a sampled vehicle true state.  Let $\mathbold{u}_{X}\in \mathcal{D}^{\mathbold{u}}$ be the projection $P_{X\to \mathbold{u}}$ of sample $X$ onto the nominal admissible input space $\mathcal{D}^{\mathbold{u}}$.
\end{definition}
Most importantly, we will only assume the existence projections $P_{X\to \mathbold{x}}$ and $P_{X\to \mathbold{u}}$, and we do not assume these projections are unique nor knowledge of their analytic forms, \textit{etc}. Now, the true dynamics \eqref{eq:true} given operating state $X$ satisfies
\begin{equation}
    \dot{\mathbold{x}}_{X} = \hat{f}(X) = f(\mathbold{x}_{X},\mathbold{u}_{X}) + \tilde{f}(X).
\end{equation}
For notation simplicity, we will drop the subscripts $X$ for $\mathbold{x}_{X}$ and $\mathbold{u}_{X}$ as we refer to the nominal state given any operating sample state $X$. 

Following many robust optimal control approaches \citep{lopez2019dynamic}, we study additive model uncertainties that satisfy the following assumption.
\begin{ass}
    The true system dynamics $\hat{f}$ can be expressed as $\hat{f} = f + \tilde{f}$ where $f$ is the nominal dynamics and $\tilde{f}$ represents all model uncertainties.
\end{ass}

\begin{figure}
\centering
    \includegraphics[width=0.99\linewidth]{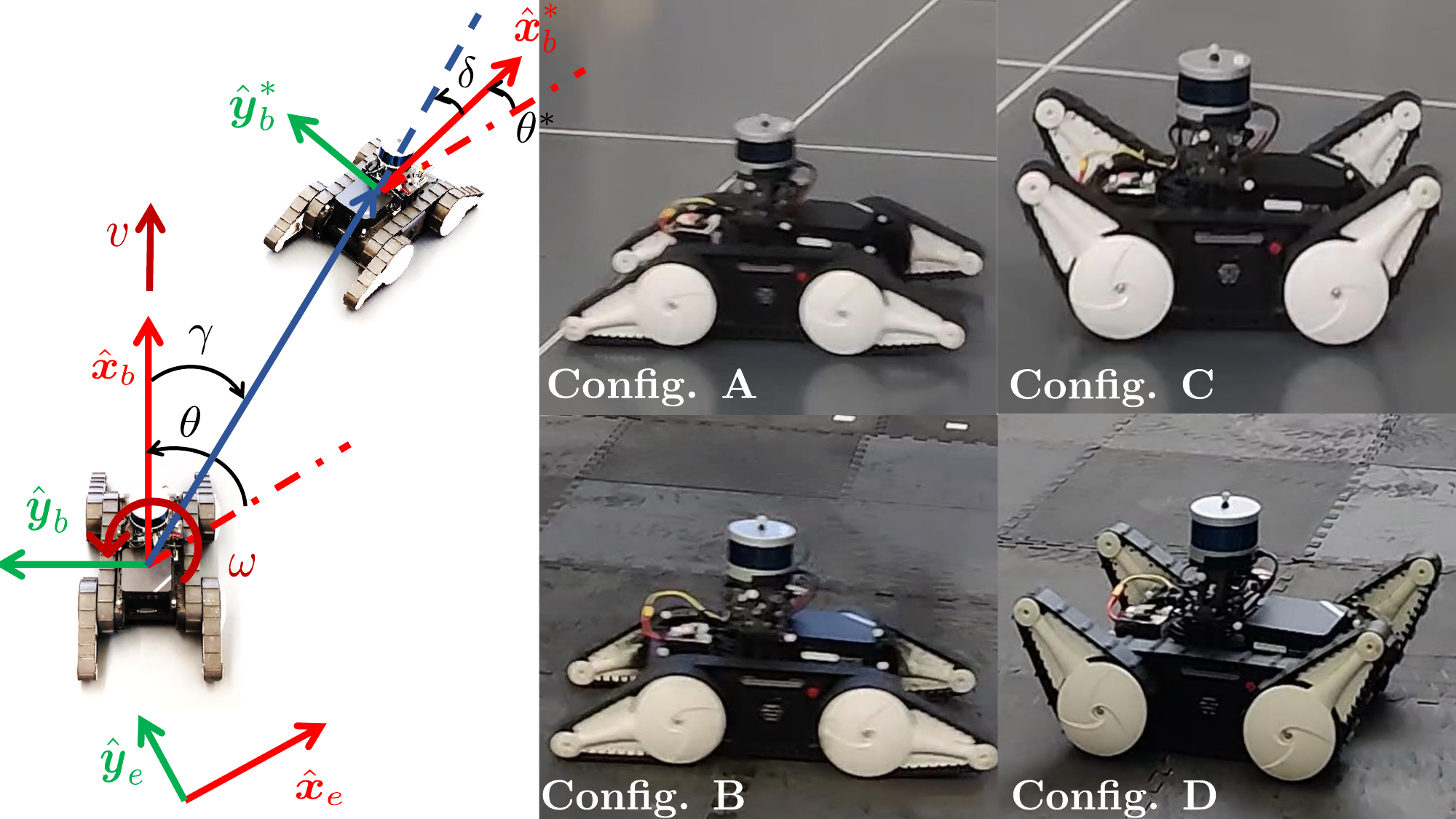}
    \caption{(Left) Definition of the polar coordinate. Subscripts $(\cdot)_b$, $(\cdot)_e$ denote the body frame and inertial frame, respectively. The pose with superscript $(\cdot)^*$ is the desired pose. (Right) Four experimentally validated scenarios. Config. A: UGV with flipper on rubber flooring. Config. B: UGV with flipper on foam flooring. Config. C: UGV without flipper on rubber flooring. Config. D: UGV without flipper foam flooring. }
    \label{fig:schematics}
\end{figure}

Our approach is based on the notion that the vehicle aims to follow a desired (denoted by $(\cdot)^d$) reference path  \mbox{$\mathbold{p}^{d}(t) = [x^{d}(t)$, $y^{d}(t)]^{T}$}, which is designed without considering obstacles or model simplifications and mismatches. Since the nominal model \eqref{eq:origin} is differentially flat, the input sequence given the desired reference $\mathbold{p}^d(t) = [x^d(t), y^d(t)]^{T}$  can be readily found using \eqref{eq:flat_input_v}, without considering obstacles. We denote the differentially flat input sequence given a position reference $\mathbold{p}^{d}(t)$ as the desired flatness-based input reference $\mathbold{u}^d = [v^d,\omega^d]$, which can be computed as
\begin{align} 
\label{eq:flat_input_v}
 v^d  \!=\!\begin{cases}
     \frac{\dot{x}^d}{\cos(\theta^d)},&\!\! \mbox{if} \sin(\theta^d)\!=\!0  \\
     \frac{\dot{y}^d}{\sin(\theta^d)},&\!\! \text{otherwise} 
 \end{cases}, \,\,
 \omega^d \! =\!
 \begin{cases}
     \frac{a^d}{v^d}, &\!\!  \mbox{if}\,v^d \neq 0  \\
     0 &\!\! \text{otherwise} 
 \end{cases},
\end{align}
\normalsize
where $a^d = \frac{-\sin({\theta}^d) \ddot{x}^d + \cos(\theta^d) \ddot{y}^d}{v^d}$ and $\theta^d = \mbox{ATAN2}(\dot{y}^d,\dot{x}^d))$.

\subsection{Robust Planning}
Suppose that a collision-avoiding planner could construct an open loop optimal state and controls sequences $(\mathbold{x}^*(t),\mathbold{u}^*(t))$ given initial condition $\mathbold{x}(0)$, that tracks the desired path $\mathbold{p}^{d}(t)$ while avoiding obstacles, for system governed by the nominal dynamics \eqref{eq:origin}. Notationally, an asterisk $(\cdot)^{*}$ denotes optimality: $\mathbold{u}^*(t)$ is a sequence of admissible control inputs for finite time or infinite time $t$ that minimizes the error in vehicle position \mbox{$\mathbold{p}(t) = [x(t),y(t)]^{T}$} relative to the reference path $\mathbold{p}^{d}(t)$, for the nominal model \eqref{eq:origin}. Note, that the planner is expected to deviate from the reference $\mathbold{p}^d(t)$ to avoid obstacles, such as in \citep{mppi_original,mppi_quadcopter,lindqvist2020nonlinear}, based on the forward propagated ego vehicle trajectory using the nominal model \eqref{eq:origin}.

Being a model-based method, the solution to the finite-time optimal control problem (FTOCP) may exhibit suboptimal performance or even safety violations of the actual vehicle in the face of modeling uncertainties and disturbances.

The standard Robust FTOCP formulation involves a minimization-maximization (min-max) optimization to construct a policy $\pi: \mathcal{D}^{x}\times \mathbb{Z} \to \mathcal{D}^{u}$ as follows~(\citep{diehl2004robust}):
\begin{align}
    \pi_k\! =\!\! &\min_{\{\mathbold{u}_{k}\}^{n_h - 1}_{k=0}}  \max_{\{\mathbold{d}_{k}\}_{k=0}^{n_h-1}}\!\sum_{k=0}^{n_h - 1}\! \mathcal{L}_{k}(\mathbold{x}_{k+1},\mathbold{u}_{k}, \mathbold{p}^d_{k+1}) \label{eq:FTOCP-cost}\\
     \mbox{s.t.} \quad&  \eqref{eq:true}, \quad   \mathbold{x}_k \in \mathcal{X},\, \mathbold{u}_k \in \mathcal{U}, \, \mathbold{d}_k \in \Omega^d, \label{eq:FTOCP-eqConst}
\end{align}
where $\mathbold{d}_{k}$ is a discrete-time realization of the model discrepancy $\tilde{f}$ given nominal state $\mathbold{x}_k$ and input $\mathbold{u}_k$. Parameters $n_h \in \mathbb{Z}_{>0}$ and $\Delta t\in \mathbb{R}_{>0}$ are the finite-time planning horizon length and time discretization step size. The stage cost $\mathcal{L}_k$ can be a general function such as the squared distance between the vehicle position and the desired reference. The solution to \eqref{eq:FTOCP-cost}-\eqref{eq:FTOCP-eqConst} can be computationally intractable~\citep{lopez2019dynamic}. Nevertheless, the resulting robust FTOCP policy may lead to overly conservative control sequences.

An alternative solution, which still provides robustness to uncertainty, is to pair an \textit{ancillary controller} $\kappa$ for disturbance rejection, with the open-loop MPC control input: 
\begin{equation} \label{eq:policy_over_arching}
    \pi = \mathbold{u}^* + \kappa(\mathbold{x},\mathbold{x}^*),
\end{equation}
where $\mathbold{u}^*(t)$ and $\mathbold{x}^*(t)$ are the open-loop input and state trajectories, respectively. The ancillary controller ensures the realized states $\hat{\mathbold{x}}$ of the uncertain system remain in a robust
control invariant (RCI) tube around the nominal trajectory $(\mathbold{x}^*,\mathbold{u}^*)$ from the initial state $\mathbold{x}_0$.

\begin{definition} \citep{lopez2019dynamic}
Let $\mathcal{D}^x$ be the set of allowable states and define the state tracking error as $\tilde{\mathbold{x}} \triangleq \mathbold{x} - \mathbold{x}^*$. The set $\Omega_{RCI} \subset \mathcal{X}$ is a \textbf{robust control invariant} (RCI) tube if there exist an ancillary controller $\kappa(\mathbold{x},\mathbold{x}^*)$ such that if $\tilde{\mathbold{x}}(t_0) \in \Omega_{RCI}$, for all realizations of the disturbance and model uncertainties, $\tilde{\mathbold{x}} \in \Omega_{RCI}$, $\forall t \geq t_0$.
\end{definition}

To complete the Robust MPC, a constraint-tightened version of the nominal MPC problem (using the nominal and unperturbed model) can be solved to generate a robust open-loop trajectory pair $(\mathbold{x}^*_{RCI},\mathbold{u}^*_{RCI})$.

\subsection{Ancillary Controller - Nominal Model}

Since the open-loop inputs $\mathbold{u}_d(t)$ or $\mathbold{u}^*(t)$ can exhibit poor tracking when applied to the actual vehicle, due to model simplifications and mismatches, we incorporate tracking error feedback using \citep{388294} as the \textit{nominal ancillary controller},
\begin{equation} \label{eq:fb_law}
  \kappa(\mathbold{e}) = \delta \mathbold{u} \triangleq \begin{cases} v = k_1 \rho \cos{\gamma}\\ 
    \omega = k_2 \gamma + k_1 \frac{\sin(\gamma)\cos(\gamma)}{\gamma}(\gamma + k_3 \delta)
 \end{cases},
\end{equation}
where $k_1$, $k_2$, and $k_3$ are positive constants. This ancillary controller drives the polar coordinates tracking error to zero with respect to the nominal model \eqref{eq:origin}. The tracking error in polar coordinates 
    \[ \mathbold{e} = [\rho, \gamma, \delta]^{T} \triangleq f_{\mathbold{e}}(\mathbold{x},\mathbold{x}^*)\]
is a function of the current vehicle state $\mathbold{x}$ and the optimal state $\mathbold{x}^*$, see Fig. \ref{fig:schematics}. 
Polar coordinate tracking error \mbox{$\rho\triangleq \sqrt{(x - x^*)^2 + (y-y^*)^2}$} is the distance from the vehicle's current $(x,y)$ position to a desired waypoint $(x^*,y^*)$. Polar tracking error \mbox{$\gamma \triangleq \tan^{-1}\frac{y - y^*}{x - x^*} - \theta$} denotes the angle between the vehicle's body-fixed $x$-axis (the longitudinal axis) and a vector that points from the vehicle's body-fixed frame origin to the desired waypoint $(x^*,y^*)$.  Similarly, \mbox{$\delta = \gamma + \theta - \theta^*$} is the analogous angle between the vehicle's body-fixed $x$-axis to the optimal $x$-axis given by $\theta^{*}$. For the nominal model, the time derivative of the polar coordinates tracking error $\mathbold{e} \triangleq [\rho, \gamma, \delta]^{T}$ satisfies:
\begin{equation}\label{eq:polar_eom}
    \dot{\mathbold{e}} =  \begin{bmatrix}
        -\cos(\gamma) &        \frac{\sin(\gamma)}{\rho} &  \frac{\sin(\gamma)}{\rho} \\
        0 & -1 & 0
    \end{bmatrix}^{T}\delta\mathbold{u} \triangleq g_{p}(\mathbold{e})\delta\mathbold{u}.
\end{equation}

The vehicle operates in the domain $\mathcal{D}^{\mathbold{e}} = \mathcal{D}_{\rho} \times \mathcal{D}_{\gamma} \times \mathcal{D}_{\delta}\subset\mathbb{R}^3$ where $\mathcal{D}_{\rho} = (dz, \rho_{max})$, $\mathcal{D}_{\gamma} = (-\pi/2,\pi/2)$, and $\mathcal{D}_{\delta} = (-\pi/2,\pi/2)$.  \footnote{The vanishingly small interval $(0,dz)$ is a ``deadzone" where the polar coordinate representation \eqref{eq:polar_eom} is ill-conditioned as $dz \to 0$.} Within domain $\mathcal{D}^{\mathbold{e}}$, the function $g_{p}(\mathbold{e})$ is bounded and Lipschitz continuous with Lipschitz constant $l_{g_p}$. 
We chose \eqref{eq:fb_law} as the ancillary controller because of the following result, given the operating domain $\mathcal{D}^{\mathbold{e}}$.

\begin{lemma}
    \citep{DeLuca2001,388294} \label{lemma:1}
    Consider system \eqref{eq:polar_eom}, control law \eqref{eq:fb_law}, and positive constants $k_1$, $k_2$ and $k_3$. The nominal closed-loop system is globally asymptotically stabilizing to $\mathbold{e}^* = 0$.
\end{lemma} 
Lemma \ref{lemma:1} can be proved by the following valid and positive definite control Lyapunov function (CLF):
\begin{equation} \label{eq:Layp}
    V(\mathbold{e}) = \frac{1}{2}(\rho^2 + \gamma^2 + k_3 \delta^2).
 \end{equation}
We here showcase several properties of $V(\mathbold{e})$ which support the main theorem in Section \ref{sec:maximum_discpreancy_tube}. For quadratic form (\ref{eq:Layp}), there exist $\alpha_1,\alpha_2 \in \mathbb{R}_{>0}$ where $\alpha_1 \|\mathbold{e}\|_2\leq V \leq \alpha_2 \|\mathbold{e}\|_2$, $\forall \mathbold{e}\in {\mathcal{D}}^{\mathbold{e}}$. Its time derivative \mbox{$\dot{V} = -k_1\cos^2(\gamma) \rho^2 - k_2 \gamma^2 = -\alpha_3(\|\mathbold{e}\|)$} is negative semi-definite. Adhering to the Lyapunov framework, there exists a class $\mathcal{K}_{\infty}$ \citep{Khalil_2002} function $\alpha_3$ such that \mbox{$\dot{V} \leq - \alpha_3(V)$}. Parameters $\alpha_1,\alpha_2 \in \mathbb{R}_{>0}$, and function $\alpha_{3}(\cdot)$ play an important role in system safety and stability when model \eqref{eq:polar_eom} is perturbed by unknown discrepancies, see Section \ref{sec:maximum_discpreancy_tube}.

 The nominal trajectory tracking control policy  
\begin{equation} \label{eq:nominal_tracking_law}
\mathbold{u} = \pi(\mathbold{e}) = \mathbold{u}^* + \kappa(\mathbold{e}) = \mathbold{u}^* + \delta \mathbold{u}
\end{equation} 
asymptotically tracks $\mathbold{x}^*$ when model mismatches are not present.  Our goal is to construct trajectories and control policies that are robust to model discrepancies and that provide probabilistic guarantees to avoid \textit{a priori} unknown obstacles.

\subsection{Modeling Discrepancies}

We use a probabilistic framework to describe model discrepancies and to define a probabilistic data-driven upper bound on the sum of all model discrepancies.

\begin{definition}
   Let function $\tilde{f}: Q \to \mathbb{R}^3$  measure the difference between the nominal model \eqref{eq:origin} and the true model  \eqref{eq:true}.  Let $X\in Q$ denote a vehicle operating state where $Q$ is the vehicle's operating state space and $\tilde{f}(X)\in \mathbb{R}^3$ denote the additive disturbance at operating state $X$. We assume that the additive disturbance random variable is drawn from a probability space $(\Omega, \mathcal{F}, \mathbb{P})$, consisting of a sample space $\Omega$, a $\sigma$-algebra $\mathcal{F}$ over $\Omega$ defining events, and a probability measure $\mathbb{P}$. We define the {\em minimal model discrepancy upper bound}, $z_{\epsilon}^{\tilde{f}}$, given a risk level $\epsilon \in [0,1]$ over $\Omega$ for function $\tilde{f}$ as:
\begin{equation} \label{eq:minimum_model_discrepancy}
        z_{\epsilon}^{\tilde{f}} \triangleq arg\min_{z\in \mathbb{R}}\{\mathbb{P}[\|\Tilde{f}(X)\| \leq z] \geq 1-\epsilon \}, \quad \forall X\in Q.
\end{equation}
\end{definition}
The value $z_{\epsilon}^{\tilde{f}}$ can be interpreted as the $(1-\epsilon)$ percentile of all possible model discrepancies norms, $\|\tilde{f}(X)\|$, for all $X\in Q$.  Such differences, i.e. residual dynamics, arise from too simplistic vehicle modeling.  The nominal model \eqref{eq:origin} disregards some robot configurations (such as different flipper angles) and terrain types (rubber and foam in our experiments) as illustrated in Fig. \ref{fig:schematics}.

The minimum value in \eqref{eq:minimum_model_discrepancy} may only be reached in the limit of infinite samples of $\tilde{f}(X)$, or knowing the probability measure of $X$ on $\Omega$, which cannot be realized in practice. Instead, we seek to identify a model discrepancy bound $Z_{\epsilon}^{\tilde{f}} \geq z_{\epsilon}^{\tilde{f}}$ given finite samples of $\mathbold{x}_{X}$, defined in \eqref{eq:projection_on_state_space}, while making no assumptions on the distribution over $\Omega$ or the explicit form or uniqueness of the projection $P_{X\to \mathbold{x}}$. That is, $Z_{\epsilon}^{\tilde{f}}$ is a model discrepancy bound, not necessarily the tightest bound over $\|\tilde{f}(X)\|$ but a realizable bound from a limited data set.

\subsection{Obstacle Avoidance}

We study the safety-critical autonomy of ground vehicles traversing through incompletely known terrains populated with stationary but \textit{a priori} unknown obstacles. We assume that a 2-D occupancy grid map ($\mathcal{O}_t $, derived from sensor data) is available to the vehicle. The map, defined in inertial coordinates, has a width and length of  $w_{map}$, $l_{map} \in \mathbb{Z}_{>0}$, respectively, and a grid resolution of $r_{map}$. Each grid cell is associated with an integer value between $0$ and $100$ that describes the cell occupancy probabilities where $0$ indicates a $0 \%$ chance of obstacle occupancy and vice versa. Unobserved cells are assigned a risk-neutral $50 \%$ occupancy probability. The map is built incrementally as the robot moves  (see Section \ref{sec:experiment}). This work focuses on the dynamic uncertainties arising from model residuals and ignores possible sensing uncertainties. Our method can be extended to include measurement uncertainties. 


Given this sensor-based grid map of the unknown environment, our framework allows a ground vehicle with terrain modeling discrepancies to track a reference trajectory within the map while having the following probabilistic safety guarantees on collision avoidance. 
\begin{equation} \label{eq:obs_avoid}
    \mathbb{P}[ \mathbold{p}_{t} \notin \mathcal{O}] > 1-\epsilon \quad  \forall t, 
\end{equation}
for any $\mathbold{p}_{t}\in \mathcal{V}_t$ where $\mathcal{V}_t$ denotes a uniform probability distribution over bounded set of $x,y$ positions occupied the vehicle geometry at time $t$. The probability statement \eqref{eq:obs_avoid} enforces collision-free given occupied space in $\mathcal{O}$ with confidence $1-\epsilon$ for all time.
\begin{rem}
Our analysis focuses on flat terrains populated with stationary obstacles. Nevertheless, our framework can be readily extended to uneven surfaces using model \eqref{eq:origin} or a higher-fidelity model for rough terrain applications, such as \citep{CHAKRABORTY20041273}, as the nominal model.  
\end{rem}
\section{Preliminaries} \label{sec:prelim}

To ensure vehicle safety in the presence of model uncertainties, we aim to infer the statistical bounds Eqn. \eqref{eq:minimum_model_discrepancy} from training data and incorporate this knowledge into the vehicle's control and planning strategies. Because differential-driven ground robots are underactuated, disturbance $\tilde{f}$ cannot be fully compensated via controller design. We formally distinguish between the components of $\tilde{f}$ that can be ``matched" by control inputs and those that remain ``unmatched" due to underactuation. Using closed-loop tracking data, we identify bounds for the matched and unmatched model discrepancies using {\em conformal prediction}, a statistical inference tool that is reviewed next. For more details, see \citep{cp_book,dixit2023adaptive}.

Using the data-driven model discrepancy bound, we formulate an FTOCP that leverages the quantified uncertainties to synthesize trajectories with probabilistic safety guarantees.  We employ a sampling-based strategy, model predictive path integral (MPPI), to solve the FTOCP. A full analysis of MPPI beyond our cursory review can be found in \citep{mppi_original,8558663}.  

\subsection{Matched and Unmatched Discrepancies}
We decompose the discrepancy $\Tilde{f}$ into a matched component that lies in the controllable subspace and an unmatched component that lies in the unmatched subspace, similar to \citep{sinha2022adaptive,1323177, pravitra2020}. 


\begin{definition} 
The additive model uncertainty $\tilde{f}$ can be decomposed into \textit{matched} ($\tilde{f}_{m}$) and \textit{unmatched} ($\tilde{f}_{um}$) discrepancies  with respect to reduced-order model \eqref{eq:origin}. The \textit{matched discrepancy} $\tilde{f}_{m}:Q\to \mathbb{R}^3$ satisfies
\begin{equation}
    \tilde{f}_{m}(X) \in \mbox{mtch}({g}(\mathbold{x}_{X})), \quad\forall \, X \in Q,
\end{equation}
where $\mbox{mtch}(g) \triangleq \mbox{span}(g_{1},g_2)(\mathbold{x})$ is the range space of the actuation matrix of system \eqref{eq:origin} such that $g_1$, $g_2$ are the columns of the actuation matrix $g(\mathbold{x})$.
The \textit{unmatched discrepancy} is defined as the function $\tilde{f}_{um}:Q \to \mathbb{R}^3$ satisfying
\begin{equation}
    \tilde{f}_{um}(X) \in \mbox{umtch}({g}(\mathbold{x}_{X})), \quad\forall \, X \in Q.
\end{equation}
where $\mbox{umtch}(g) \triangleq \{\mathbold{x} \in \mathcal{D}^{\mathbold{x}} \,|\, rank(g(\mathbold{x})) < 2\}$ is the null space of the actuation matrix.
For all samples $X\in Q$ and their respective reduced-order state $\mathbold{x}_{X}\in \mathcal{D}^{\mathbold{x}}$, the summation $\tilde{f}(X) = \tilde{f}_{m}(X)  + \tilde{f}_{um}(X)$ holds.
\end{definition}

\subsection{Conformal Prediction}
Conformal prediction (CP) provides model- and assumption-free uncertainty quantification to black-box models\citep{cp_book,angelopoulos2021gentle}. Conformal prediction produces sets that are guaranteed to contain the desired ground truth on any pre-trained model while satisfying a user-defined probability. Our pre-trained model is the nominal model \eqref{eq:origin}. We use CP to calibrate a model discrepancy \mbox{$Z_{\epsilon}^{\tilde{f}} \geq z_{\epsilon}^{\tilde{f}}$} with a user-defined probability $\epsilon \in [0,1]$ without needing to make assumptions on the probability distribution over the discrepancy sample space.

Let $Y_0, \hdots, Y_{n-1}$ be $n$ exchangeable random variables where each $Y_i \in \mathbb{R}$ is a {\em nonconformity score}. The nonconformity score is typically chosen to express the difference between the pre-training model of an unknown system and calibration data (discrepancy ground truth) from that unknown system. For calibrating model discrepancies $\tilde{f}$ in \eqref{eq:true}, the nonconformity score can be a metric measuring its magnitude such as $Y^i \triangleq \|\hat{f}(X_i) - f(X_i)\| = \|\tilde{f}(X_i)\|$. The conformal prediction algorithm outputs a value $\overline{Y}_{\epsilon}: \mathbb{R}^{n}\times [0,1] \to \mathbb{R}$ such that the inequality $\mathbb{P}[Y\leq \overline{Y}_{\epsilon}] \geq 1 - \epsilon$~holds over the sample space of $Y$. Intuitively speaking, a large $Y_i$ can be interpreted as a large model discrepancy, indicating that there is a poor matching between the nominal and true model. 


Conformal prediction is an uncertainty quantification framework built on top of empirical statistics. Let the cumulative distribution function $F_{X} \triangleq \mathbb{P}(Y(X)\leq \overline{Y}_{\epsilon}) = p$. Let $Q: [0,1] \to \mathbb{R}$ be a Quantile function, the $(1-\epsilon)^{th}$ quantile returns the value $\overline{Y}_{\epsilon}$ such that $Y \leq \overline{Y}_{\epsilon}$ for all $X\in \Omega$ as $
    Q(1-\epsilon) = F_{X}^{-1}(1-\epsilon)$.


Tibshirani \textit{et. al.}\citep[Lemma 1]{tibshirani2019conformal} relates that the desired $\overline{Y}_{\epsilon}$ is equivalently the $(1-\epsilon)^{th}$ quantile of the empirical distribution formed by nonconformity scores $\{Y_0,\hdots,Y_{n-1}$, $\infty\}$ with an additive $\infty$. To calculate the $(1-\epsilon)^{th}$ quantile, let $\{Y_{(0)},\hdots,Y_{(n)}\}$ be a non-decreasing sorting of $\{Y_0,\hdots,Y_{n-1}, \infty\}$ where $Y_{(i)} \leq Y_{(i+1)}, \forall i\in \mathbb{Z}_{0}^{n}$, also known as the the order statistics of $\{Y_i\}_{i=0}^{n}$. The integer $q_{\epsilon} \triangleq \lceil (n+1)(1-\epsilon)\rceil$ is the order index that corresponds to the $(1-\epsilon)$~percent confidence level empirically on the calibration data set, where the multiplier $(n+1)$ is the finite sample adjustment. The $(1-\epsilon)^{th}$ quantile, $\overline{Y}_\epsilon \triangleq Y_{(q_{\epsilon})}$ defines the desired $1-\epsilon$ prediction region satisfying the following probability statement:
\begin{equation} \label{eq:cp_probability_statement}
    \mathbb{P}[\|\tilde{f}(X)\| \geq \overline{Y}_\epsilon] \leq \epsilon
\end{equation}
for any random variable $X\in Q$ and $\overline{Y}_{\epsilon} = Z_{\epsilon}^{\tilde{f}}$ as desired. The set $C_{\tilde{f}}(X) \triangleq \{\mathbold{y}\in \mathbb{R}^{2}\, |\, \|\mathbold{y}\| \leq \overline{Y}_{\epsilon}\}$ is the $(1-\epsilon)$ confidence conformal prediction set. 

For the tightness of the conformal probabilistic upper bound, Sadinle \textit{et. al.} \citep{sadinle2019least} shows the conformal prediction set $C_{\tilde{f}}$, given $\epsilon \in [0,1]$, achieves the smallest average set of all possible prediction schemes $\mathcal{C}$ that offer the desired coverage guarantee if the sample nonconformity scores $\{Y_0, \cdots, Y_{N-1}\}$ reflects the true conditional probability:
\begin{equation} \label{eq:smallest_cp}
    \min_{C\in \mathcal{C}} \mathbb{E}[C(X)], \quad \mbox{subject to} \, \eqref{eq:cp_probability_statement}.
\end{equation}


\subsection{Model Predictive Path Integral (MPPI)}
MPPI is a sampling-based method, leveraging a duality condition from information theory, to obtain local optimal controllers given potentially nonlinear cost and constraints. Consider input sequence $\{\mathbold{u}_{i}\}_{i=0}^{t}$ where $\mathbold{u}\in \mathbb{R}^{n_{\mathbold{u}}}$, we define $\mathbold{v}_i \triangleq \mathbold{u}_i + \mathbold{\delta}_i$, $\forall i\in \{0,\hdots, t\}$ as the sequence of {\em perturbed inputs} where $\mathbold{\delta}_{i}$ are zero mean Gaussian input perturbations  $\{\mathbold{\delta}_i\}_{i=0}^{t-1} \in \mathcal{N}(0, \Sigma_{\mathbold{u}})$. We denote the input trajectories to the algorithm as the following sequences
\begin{eqnarray}
    V_{t-1} &\triangleq &  \left[\mathbold{v}_{0},\mathbold{v}_0,\cdots, \mathbold{v}_{t-1} \right] \in \mathbb{R}^{n_{\mathbold{u}}\times t},\\
    \Delta_{t-1} &\triangleq &  \left[\mathbold{\delta}_{0},\mathbold{\delta}_1,\cdots, \mathbold{\delta}_{t-1} \right] \in \mathbb{R}^{n_{\mathbold{u}} \times t},\\
    U_{t-1} &\triangleq &  \left[\mathbold{u}_{0},\mathbold{u}_1,\cdots, \mathbold{u}_{t-1} \right] \in \mathbb{R}^{n_{\mathbold{u}}\times t}.
\end{eqnarray}

Following the result from \citep{8558663}, the probability density function for $V_t$, denoted as $d_{V|U,\Delta}$, is
\begin{equation}
    d_{V|U,\Delta} = \left({(2\pi)^{n_{\mathbold{u}}}|\Sigma_{\mathbold{u}}|}\right)^{\frac{-t}{2}} \mbox{exp}\left(-\frac{1}{2} \sum_{i=1}^{t}\mathbold{\delta}_i^{T} \Sigma_{\mathbold{u}}^{-1}\mathbold{\delta}_{i}\right),
\end{equation}
with corresponding sample space, denoted as ${\Omega}_{U,\Sigma}$. Similarly, we denote the probability distribution corresponding to $U_t = 0$ as $\Omega_{0,\Sigma}$, known as the base distribution.
Consider the stochastic trajectory optimization problem, adopted from \citep{8558663}:
\begin{equation} \label{eq:stochastic_traj_opt}
    U_{t-1}^{*} = \min_{\begin{smallmatrix}
        \mathbold{u}_{i} \in \mathcal{U} \\
        \forall i \in \mathbb{Z}_{1}^{t}
    \end{smallmatrix}}  \mathbb{E}_{\Omega(U,\Sigma)} \Big[\mathcal{L}_{T}(\mathbold{x}_t) + \sum_{i=0}^{t-1}\mathcal{L}(\mathbold{x}_i, \mathbold{u}_i)\Big],
\end{equation}
where $\mathcal{L}_{T}$ and $\mathcal{L}$ respectively denote terminal and stage costs. Let $\psi_{\mathbold{x}}(t,{V}_{t-1},\mathbold{x}_0)$ be the flow of system \eqref{eq:origin} generated by applying input sequence $V_t$ from initial condition $\mathbold{x}_0$. 


From an information theoretic perspective, optimization \eqref{eq:stochastic_traj_opt} can be converted into a probability-matching problem \citep{8558663}. Define the \textit{Free Energy} $\mathcal{F_E}$ as 
\begin{align*}
    \mathcal{F_E} &=  \log \left( \mathbb{E}_{\Omega_{0,\Sigma}}\Bigg[ \mbox{exp}\left(- \frac{C_{\mathbold{x}}(V_{t-1},\mathbold{x}_0)}{\lambda}\right) \Bigg]\right),
\end{align*}
where $\lambda > 0$ is commonly referred to as the \textit{inverse temperature}, and the function $C_{\mathbold{x}}(V_{t-1},\mathbold{x}_0) = C_{\mathbold{x}}$ outputs the cost of the trajectory generated using input sequence $V_{t-1}$ from initial state $\mathbold{x}_0$. Using the KL-Divergence properties between two probability distributions that are \textit{absolutely continuous} with each other, Williams \textit{et. al.}\citep{8558663} showed that a lower bound to the optimal control problem \eqref{eq:stochastic_traj_opt} can be obtained using the free energy:
\begin{align}
    -\lambda \mathcal{F}_{\mathcal{E}} \leq \mathbb{E}_{\Omega_{U,\Sigma}}\left[\mathcal{L}_{T}(\mathbold{x}_{t}) + \sum_{i=0}^{t-1}\mathcal{L}(\mathbold{x}_i,\mathbold{u}_i)\right].
\end{align}


Summarizing the result in \citep{mppi_original,7487277}, the following iterative update law, building upon importance sampling,
\begin{align}
    U_{t}^{*} &= U_{t} + \sum_{j=1}^{N_{sample}} w(\Delta^j_{t}) \Delta_{t}^j, \label{eq:is_1}\\
    w(\Delta_{t}^j) &=  \frac{1}{\eta} \mbox{exp}\left({-\frac{C_{\mathbold{x}}}{\lambda} - \sum_{i=0}^{t-1} \mathbold{u}_i^{T}\Sigma_{\mathbold{u}}^{-1} (\mathbold{u}_i + 2\mathbold{\delta}_i^j)}\right), \label{eq:is_2}
\end{align}
produces a locally optimal solution to the problem \eqref{eq:stochastic_traj_opt}.
The weighting term $w\in \mathbb{R}$ characterizes the ``importance" of each sampled input perturbation $\{\Delta_{t}^{j}\}_{j=1}^{N_{sample}}$, where $N_{sample}$ is the number of sampled input perturbations. The parameter $\eta$ can be approximated as
\begin{equation*}
    \eta \approx \sum_{i=1}^{N_{sample}}\mbox{exp}\left({-\frac{C_{\mathbold{x}}}{\lambda} - \sum_{i=0}^{t-1} \mathbold{u}_i^{T}\Sigma_{\mathbold{u}}^{-1} (\mathbold{u}_i + 2\mathbold{\delta}_i^j)}\right)
\end{equation*}
using Monte-Carlo estimation \citep{mppi_original}. 

In general, MPPI is a sampling-based alternative to optimization-based MPC algorithms, flexible to general cost functions and system dynamics. The aggregation law \eqref{eq:is_1}-\eqref{eq:is_2} is highly parallelizable for online receding-horizon planning, as demonstrated in UGV racing applications \citep{mppi_original}.


\section{Data-Driven Discrepancy Identification} \label{sec:dd_discrepancy_identification}
This section summarizes the data-driven identification of a model discrepancy bound $Z_{\epsilon}^{\tilde{f}}$, where $\tilde{f}$ captures model mismatches in \eqref{eq:true} from sources like slipping, skidding, change in surface traction properties, input delays, etc. Discrete-time data is captured from the continuous system at sampling interval $\Delta T$.  Time is indicated by an integer:  variable $\mathbold{x}_i$ denotes the value of state $\mathbold{x}$ at time $t_i=i\Delta T$. $N_T$ tuples of training data are gathered from running the system with the nominal tracking control law \eqref{eq:nominal_tracking_law} and each tuple has the form 
\begin{equation} \label{eq:tuple_training}
 S_{i} = \{\mathbold{x}_{i-1},\hat{\mathbold{x}}_i, \mathbold{x}^*_i, \mathbold{u}_{i-1},\mathbold{u}^*_{i-1}\},
\end{equation}
where $\mathbold{x}_{i-1}$ is the state at time $t_{i-1}$, and the pair $(\mathbold{x}^*_i, \mathbold{u}^*_{i-1})$ is the optimal state and the optimal input at $t_i$ with respect to the cost in \eqref{eq:stochastic_traj_opt}. This data set is collected from multiple vehicle trajectories that result from an optimal trajectory planner (e.g., by a MPC-based planner for dynamics \eqref{eq:origin}) of the form $(\mathbold{x}^*_{i}, \mathbold{u}^*_{i-1})$ given current state $\mathbold{x}_{i-1}$. 

The actual control input $\mathbold{u}_{t-1}$ applied to the vehicle in the interval $[t_{i-1},t_i]$ is the sum of the optimal input $\mathbold{u}^*_{t-1}$ and and a correction term, $\delta \mathbold{u}_{t-1}$, applied by the ancillary controller $\eqref{eq:fb_law}$, given the polar coordinate tracking error: $\mathbold{u}_{i-1} = \mathbold{u}^*_{i-1} + \delta \mathbold{u}_{i-1} = \mathbold{u}^*_{i-1} + \kappa(\mathbold{e}_{i-1})$. $\hat{x}_i$ is the system's actual state arising from the control input $\mathbold{u}_{i-1}$  when it is held constant during $t\in [t_{i-1},t_i]$. 
The propagation of state $\mathbold{x}_{i-1}$ with input $\mathbold{u}_{i-1}$ through nominal model \eqref{eq:origin}, might be different from the measured $\hat{\mathbold{x}}_{i-1}$ because of unmodeled discrepancy $\tilde{f}$. Assuming that the optimal planner and ancillary controller are recomputed at each time step, we have for the $j^{th}$ tuple $S_{j}$, ${\mathbold{x}}_{j-1} = \hat{\mathbold{x}}_j$ holds. Figure \ref{fig:conformal_prediction_overview} overviews the proposed identification process from training data. 

\begin{figure*}
\centering
\centering
    \includegraphics[width=0.99\linewidth]{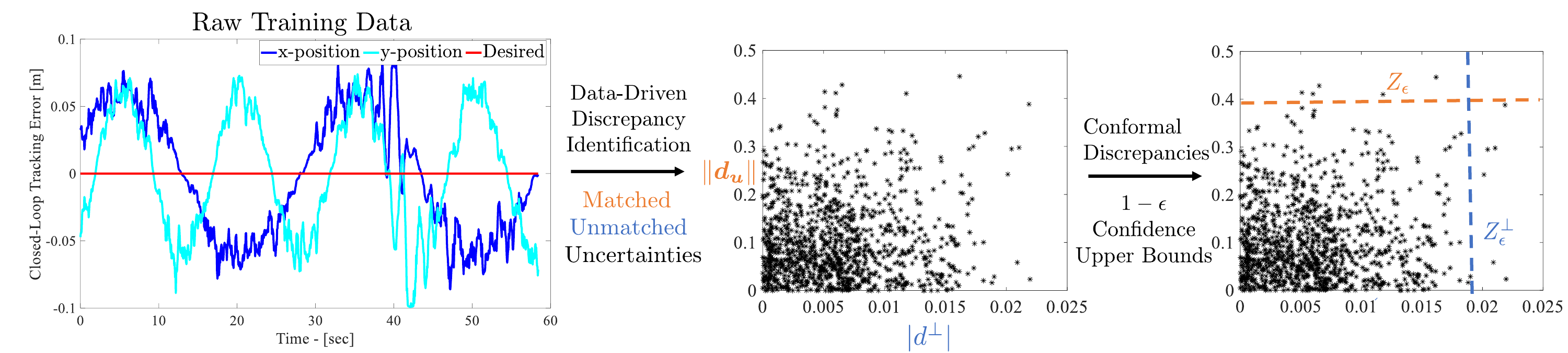}
    \caption{An outline of the proposed data-driven discrepancy identification method. From training data $\mathcal{S}_i$ (left panel), we first separate the closed-loop tracking errors into matched and unmatched model discrepancies (center panel). Given a user-specified risk tolerance $\epsilon$, conformal prediction is applied to obtain their probabilistic upper bounds (right). }
    \label{fig:conformal_prediction_overview}
\end{figure*}

\subsection{Discrepancy Identification}

To identify the matched and unmatched input discrepancies, we first compute the nominal model tracking error $\mathbold{e}_{i-1}$ using $\mathbold{x}_{i-1}$ and $\mathbold{x}^*_i$ at time $t_{i-1}$. Then, the feedback (measured) tracking error at time $t_i$, denoted as $\hat{\mathbold{e}}_i$, is a function of $\hat{\mathbold{x}}_i$ and $\mathbold{x}^*_i$.

Using \eqref{eq:polar_eom}, we calculate the polar coordinate tracking error $\mathbold{e}_L$ from the nominal error dynamics \eqref{eq:polar_eom} at $t_i$ given initial condition $\mathbold{e}_{i-1}$ and assuming that the ancillary controller $\delta \mathbold{u}_{i-1}$ satisfies
\begin{equation}
\mathbold{e}_i = \int_{t_{i-1}}^{t_{i}} g_p({\mathbold{e}}(\tau))\delta\mathbold{u}_{i-1} d\tau + {\mathbold{e}}_{i-1},
\end{equation}
where $\delta \mathbold{u}_{i-1}$ is zero-order hold during the interval $[t_{i-1},t_i]$.

Let $\delta{\hat{\mathbold{e}}}_{i\Delta T} \triangleq \hat{\mathbold{e}}_i - \mathbold{e}_i$ be the measured tracking error between the true system \eqref{eq:true} against the nominal system \eqref{eq:origin} given the same initial condition, $\mathbold{e}_{i-1}$, and the same input over the interval, $\delta \hat{\mathbold{e}}_{i\Delta T}$ which can be decomposed as follows
\begin{align} 
\delta{\hat{\mathbold{e}}}_{i\Delta T} \!&=\!\! \int_{t_{i-1}}^{t_{i}}\!\!\! \left(\hat{f}(\hat{\mathbold{e}}(\tau)) \!-\! g_p({\mathbold{e}}(\tau))\delta\mathbold{u}_{i-1}\right) d\tau \!+\!\!\
\underbrace{\delta\mathbold{e}_{(i-1)\Delta T}}_{=0} \nonumber \\
&= \int_{t_{i-1}}^{t_i} g_p(\mathbold{e}(\tau))\mathbold{d}_{\mathbold{u}} d
    \tau   + \int_{t_{i-1}}^{t_i} \mathbold{d}_{\mathbold{u}^{\perp}}d\tau, \label{eq:error_discrepancy}
\end{align}
where $\mathbold{d}_{\mathbold{u}}\in \mathbb{R}^2$ represents the matched input discrepancy. The error $\mathbold{d}_{\mathbold{u}^{\perp}} = d^{\perp} \hat{\mathbold{n}}^{\perp}_{g_p}\in\mathbb{R}^3$ lies in the null space of $g_p$, which is spanned by the unit vector $\hat{\mathbold{n}}^{\perp}_{g_p}$. Parameter $d^{\perp} \in \mathbb{R}$ is the magnitude of the unmatched disturbance. By numerical differentiation with $\Delta T \to 0$, we can identify the perturbed polar coordinate error dynamics in continuous time:
\begin{equation}
    \dot{\hat{\mathbold{e}}} = \dot{\mathbold{e}} + \delta\dot{\mathbold{e}}= g_{p}(\hat{\mathbold{e}})(\delta\mathbold{u} + \mathbold{d}_{\mathbold{u}}) + d_{\perp} \hat{\mathbold{n}}^{\perp}_{g_p}. \label{eq:preturbed_tracking_eqn}
\end{equation}
with initial condition $\mathbold{e}(0) = \hat{\mathbold{e}}(0)$. For a physical robotic system, we assume that the dynamics discrepancy is bounded.

\begin{ass} \label{ass:assump1}
    The input (matched) disturbance $\mathbold{d}_{\mathbold{u}}:\hat{\mathcal{D}}^{\mathbold{e}} \to \mathbb{R}^2$ and the unmatched drift ${d}_{\perp}:\hat{\mathcal{D}}^{\mathbold{e}} \to \mathbb{R} $ are bounded.
\end{ass} 

For safety, the system must be robust up to a user-defined risk tolerance against the worst-case discrepancies. 
Therefore, we aim to identify matched and unmatched discrepancy bounds for both $\|\mathbold{d}_{\mathbold{u}}\|$ and $|{d}_{\perp}|$ using conformal prediction from the finite training data. Despite \citep{sadinle2019least} addresses that conformal prediction can produce the desired minimum model discrepancy upper bounds \eqref{eq:minimum_model_discrepancy}, the assumption that the sampled nonconformity score satisfying \eqref{eq:smallest_cp} is not necessarily valid. Nevertheless, conformal prediction set coverage guarantees to hold regardless of the accuracy of the sampled nonconformity scores computed assuming zero model discrepancy, as they are designed for uncalibrated models \citep{ren2023robots}. 

Specifically, we seek to identify the probabilistic bounds on the matched and unmatched disturbances, satisfying:
\begin{align} \label{eq:desired_bounds_u}
    &Z_{\epsilon} \geq \min_{z\in \mathbb{R}}\{\mathbb{P}_{\hat{\Omega}}[|\|\mathbold{d}_{\mathbold{u}}(E)\| \leq z] \geq 1-\epsilon \},\\\label{eq:desired_bounds_uperp}
   &Z_{\epsilon}^{\perp} \geq \min_{z\in \mathbb{R}}\{\mathbb{P}_{\hat{\Omega}}[|{d}_{{\perp}}(E)| \leq z] \geq 1-\epsilon \},
\end{align}
where $\epsilon \in (0,1)$ is the user-defined risk for all random variables $E\in \hat{\Omega}$. Physically speaking, $E$ denotes a random sample of polar tracking error over the set of all possible polar tracking errors by ${\Omega}$. The following subsection details how to find $\|\mathbold{d}_{\mathbold{u}}\|$ and $|{d}_{\perp}|$ from the training data and how to calculate $Z_{\epsilon}$ and $Z_{\epsilon}^{\perp}$ using conformal prediction.

\subsection{Conformal-Driven Discrepancies}
 For each training state-input tuple, we calculate a tracking error $\delta\hat{\mathbold{e}}_{i\Delta T}$,  and their the discrepancy bounds $\|\mathbold{d}_{\mathbold{u}}^{i}\|$ and 
 $\|d_{\perp}^i\|$ given samples $i$ by solving the following equations in the given order:
\begin{align} \label{eq:du}
    &\|\mathbold{d}_{\mathbold{u}}^{i}\|\! =\! \max_{\mathbold{d}_{u} \in \mathbb{R}^2} \!
         \Bigg\| |\delta{\hat{\mathbold{e}}}_{i\Delta T}| \!- \Big|\int_{t_{i-1}}^{t_i} g_{p}(\mathbold{e}(\tau))d\tau \mathbold{d}_{\mathbold{u}}\Big|\Bigg\|_2 , \\ \label{eq:du_perp}
    &|{d}_{\perp}^i| = \Bigg\| |\delta{\hat{\mathbold{e}}}_{i\Delta T}|- \Big|\int_{t_{i-1}}^{t_i} g_{p}(\mathbold{e}(\tau))d\tau \mathbold{d}_{\mathbold{u}}^i\Big| \Bigg\|_2.
\end{align}
Since input disturbance $\mathbold{d}_{\mathbold{u}}$ can be actively compensated by control input $\mathbold{u}$, we identify the largest possible input discrepancy and allocate the remaining model mismatches to the unmatched drift.

Assume that the training data is sufficiently rich.
We randomly sub-sample $L$ values of $\|\mathbold{d}_{\mathbold{u}}^i\|$ and $\|\mathbold{d}_{\mathbold{u}^{\perp}}^i\|$ without replacement where $0 \ll L<N_T$. Algorithm \ref{alg:1} uses training data set $S_{N_T}$ to identify a $(1-\epsilon)$ upper bound $(Z_{\epsilon},Z^{\perp}_{\epsilon})$ defined in \eqref{eq:desired_bounds_u} and \eqref{eq:desired_bounds_uperp} using the non-conformity scores $\{\|\mathbold{d}_{\mathbold{u}}^i\|\}_{i=0}^{L-1}$ and $\{\|\mathbold{d}_{\mathbold{u}^{\perp}}^i\|\}_{i=0}^{L-1}$. 
\begin{algorithm}[!h]
\small
\caption{Conformal-Driven Discrepancies}
\KwData{Risk $\epsilon$, sub-sample $L$, training set $\{S_{i}\}_{i=1}^{N_T}$.}

\KwResult{$Z_{\epsilon}$, $Z_{\epsilon}^{\perp}$ defined by \eqref{eq:desired_bounds_u}-\eqref{eq:desired_bounds_uperp}.}


Compute index $q_{\epsilon} = \lceil (L+1)(1-\epsilon) \rceil$.

\For{$j$ from $1$ to $L$}{
Sample $\tilde{\mathbold{e}}$ without replacement and use \eqref{eq:du} and \eqref{eq:du_perp} to compute $\|\mathbold{d}_{\mathbold{u}}^j\|$, $|d_{\perp}^j|$.}

Add $\infty$ to the non-conformity scores $\{\|\mathbold{d}_{\mathbold{u}}^i\|\}_{j=0}^{L-1}$ and $\{\|\mathbold{d}_{\mathbold{u}^{\perp}}^j\|\}_{j=0}^{L-1}$.

Sort $\{\|\mathbold{d}_{\mathbold{u}}^j\|\}_{j=1}^{L+1}$ and $\{|d_{\perp}^j|\}_{j=1}^{L+1}$ in a non-decreasing orders. 

$Z_{\epsilon}$, $Z_{\epsilon}^{\perp}$ are the $q_{\epsilon} ^{th}$ smallest value in the sorted lists.
\label{alg:1} 
\normalsize
\end{algorithm}


Implicitly, the choice of the non-conformity score assumes the discrepancy noise is uncolored and zero mean. Such an assumption always leads to more conservative estimates of the discrepancy upper bound for safety-critical applications. If the noise is known to be colored, a new nonconformity score that offsets the expected discrepancies, $\|\tilde{f}(X) - \mathbb{E}_{\Omega}[\tilde{f}(X)]\|_2$,  could lead to less conservative identified upper bounds. 

In summary, we introduce a distribution-free method to identify the probabilistic upper bounds for closed-loop tracking input discrepancies $Z_{\epsilon}$ and the unmatched drift $Z^{\perp}_{\epsilon}$. The following section augments the ancillary controller using the data-driven discrepancy bounds to identify the associated maximum trajectory tracking error.

\section{Maximum Tracking Error Tubes} \label{sec:maximum_discpreancy_tube}
\noindent In addition to the identified error dynamics $\eqref{eq:preturbed_tracking_eqn}$, we also introduce the following perturbed variations of the nominal error dynamics $\eqref{eq:polar_eom}$,
\begin{align} 
    \dot{\tilde{\mathbold{e}}} &= g_{p}(\tilde{\mathbold{e}})(\delta\mathbold{u}+ \mathbold{d}_{\mathbold{u}}), \label{eq:perturbed_dyn_lemma1}\\
    \dot{\overline{\mathbold{e}}} &= g_{p}(\overline{\mathbold{e}})\delta\mathbold{u} + {d}^{\perp} \hat{\mathbold{n}}_{g_p}^{\perp},  \label{eq:perturbed_dyn_lemma2}
\end{align}
where $\mathbold{e}(0) = \tilde{\mathbold{e}}(0) = \overline{\mathbold{e}}(0) =\hat{\mathbold{e}}(0)$. The perturbed dynamics \eqref{eq:perturbed_dyn_lemma1} corresponds to the case where there are no unmatched discrepancies, i.e. $d^{\perp}(t) = 0, \forall t$. Similarly, the dynamics of $\overline{\mathbold{e}}$ corresponds to the case where there is no matched uncertainty along the controllable subspace, i.e. $\mathbold{d}_{\mathbold{u}}(t) = 0, \forall t$.
To study the effect of the discrepancies, we introduce the following lemmas, which leverage ideas from input-to-state stability.
\begin{lemma} \label{lem:1}
    Consider the perturbed error dynamics \eqref{eq:perturbed_dyn_lemma1}. Assume a positive constant $Z_\epsilon$ exists and satisfies Eq.  \eqref{eq:desired_bounds_u} for $t\in [0,\Delta T]$ as $\epsilon \to 0$. Define the set:
    \begin{equation} \label{eq:omega_iss}
        \Omega_{ISS} = \Big\{\tilde{\mathbold{e}} \in \hat{\mathcal{D}}^{\tilde{\mathbold{e}}} \,\Big|\, V(\tilde{\mathbold{e}}) \leq Z_{\epsilon}^2/4\Big\},
    \end{equation}
where $V$ is the positive definite function \eqref{eq:Layp}.
Consider the augmented ancillary control law,
\begin{equation} \label{eq:kappa_iss}
    \kappa_{ISS}(\tilde{\mathbold{e}}) = \kappa(\tilde{\mathbold{e}}) - \frac{1}{\lambda_1}g_p(\tilde{\mathbold{e}})^{T} \tilde{\mathbold{e}}.
\end{equation} 
For any $\tilde{\mathbold{e}}(0) \in \Omega_{ISS}$, if there exists a $\lambda_1 >  0$ such that applying the augmented control law \eqref{eq:omega_iss} yields the inequality $-\alpha_{3}(\alpha_2^{-1}(V(\tilde{\mathbold{e}}))) + \lambda_1 V(\tilde{\mathbold{e}}) \leq -\tilde{\alpha}_3({V}(\tilde{\mathbold{e}}))$, where $\tilde{\alpha}_3$ is a class $\mathcal{K}_{\infty}$ function, for all $\tilde{\mathbold{e}} \in \Omega_{ISS}$, then $\psi_{\tilde{\mathbold{e}}}(t,\kappa_{ISS}(\tilde{\mathbold{e}}),\tilde{\mathbold{e}}(0)) \in \Omega_{ISS}(t)$ \footnote{$\psi_{\tilde{\mathbold{e}}}(t,\kappa_{ISS}(\tilde{\mathbold{e}}),\tilde{\mathbold{e}}(0))$ is the solution (flow) of $\tilde{\mathbold{e}}$ governed by eqn. \eqref{eq:perturbed_dyn_lemma1} under the augmented control law $\kappa_{ISS}$ with the initial value $\psi(0) = \tilde{\mathbold{e}}(0)$.} for all $t\in (0, \Delta T)$.
\end{lemma}
\begin{proof}
    From the definition of the positive definite function $V$~\eqref{eq:Layp} and with the augmented controller $\kappa_{ISS}$, we have
    \begin{align*}
        \dot{V} &\leq -\alpha_3(\|\tilde{\mathbold{e}}\|) -\frac{1}{\lambda_1}\tilde{\mathbold{e}}^{T}g_p(\tilde{\mathbold{e}})g_p(\tilde{\mathbold{e}})^{T} \tilde{\mathbold{e}} +  \tilde{\mathbold{e}}^{T}g_p(\tilde{\mathbold{e}})\mathbold{d}_{\mathbold{u}} \\
        &\leq -\alpha_3(\|\tilde{\mathbold{e}}\|)  -\frac{1}{\lambda_1}\|g_p(\tilde{\mathbold{e}})^{T}\tilde{\mathbold{e}}\|_2^2 +  |\tilde{\mathbold{e}}^{T}g_p(\tilde{\mathbold{e}})\mathbold{d}_{\mathbold{u}}| \\
        &\leq -\alpha_3(\|\tilde{\mathbold{e}}\|) -\frac{1}{\lambda_1}\Big\|g_p^{T}\tilde{\mathbold{e}} + \frac{\lambda_1}{2} \mathbold{d}_{\mathbold{u}}\Big\|^2_2 +  \frac{\lambda_1 (Z_{\epsilon})^2}{4}\\
        &\leq  -\alpha_3(\|\tilde{\mathbold{e}}\|) + \frac{\lambda_1 (Z_{\epsilon})^2}{4}. 
    \end{align*}
Given $\lambda_1> 0$ exists,  the following inequalities hold 
\begin{align*}
    \dot{V} &\leq -\alpha_{3}(\|\tilde{\mathbold{e}}\|) + \lambda_1 V(\tilde{\mathbold{e}}) \\&\leq -\alpha_{3}(\alpha_2^{-1}(V(\tilde{\mathbold{e}})) + \lambda_1 V(\tilde{\mathbold{e}}) \leq -\tilde{\alpha}_3(V(\tilde{\mathbold{e}}))
\end{align*} 
given perturbed dynamics \eqref{eq:perturbed_dyn_lemma1}. With $V(\mathbold{e}(0)) \leq \frac{Z_{\epsilon}^2}{4}$, we have $V(\tilde{\mathbold{e}}(t)) \leq \frac{Z_{\epsilon}^2}{4}, \forall t$, i.e. $\psi_{\tilde{\mathbold{e}}}(t,\kappa_{ISS}(\tilde{\mathbold{e})},\tilde{\mathbold{e}}(0))\in \Omega_{ISS}$. 
\end{proof}

It is important to recognize the set $\Omega_{ISS}$ is a Robust Control Invariant (RCI) set. Lemma \ref{lem:1} only considers the case where all disturbances lie in the controllable subspace. For the complementary case governed by \eqref{eq:perturbed_dyn_lemma2}, we leverage the Gronwall-Bellman inequality to deduce the following result.
\begin{lemma} \label{lemma:2}
Under the nominal dynamical system \eqref{eq:polar_eom}, the perturbed dynamical system \eqref{eq:perturbed_dyn_lemma2} controlled by the nominal feedback control law $\mathbold{u} = \kappa(\mathbold{e})$ in Eqn. \eqref{eq:fb_law} has the closed-loop dynamics
\begin{align*}
    \dot{\mathbold{e}} &= g_p(\mathbold{e})\kappa(\mathbold{e}) = f_{cl}(\mathbold{e}),\\
    \dot{\overline{\mathbold{e}}} &= g_p(\overline{\mathbold{e}})\kappa(\overline{\mathbold{e}}) + {d}^{\perp} \hat{\mathbold{n}}_{g_p}^{\perp} = {f}_{cl}(\overline{\mathbold{e}}) +  {d}^{\perp} \hat{\mathbold{n}}_{g_p}^{\perp}.
\end{align*}
Let $Z_{\epsilon}^{\perp}\in \mathbb{R}_{>0}$ satisfy \eqref{eq:desired_bounds_uperp}. Suppose there exists a Lipschitz constant $l_{cl}$ with respect to the domain $\hat{\mathcal{D}}^{\mathbold{e}}$ for function $f_{cl}$. 
Let $\psi_{\mathbold{e}}(t,\kappa(\mathbold{e}),\mathbold{e}(0))$ and $\psi_{\overline{\mathbold{e}}}(t,\kappa(\overline{\mathbold{e}})),\overline{\mathbold{e}}(0))$ be the flows of the systems \eqref{eq:polar_eom} and \eqref{eq:perturbed_dyn_lemma2}. The following inequality holds for all $t\in [0, \Delta T]$:
\begin{multline} \label{eq:upperbound}
    \| \psi_{\mathbold{e}}(t,\kappa(\mathbold{e}),\mathbold{e}(0)) - \psi_{\overline{\mathbold{e}}}(t,\kappa(\overline{\mathbold{e}})),\overline{\mathbold{e}}(0))\| \leq \\Z^{\perp}_{\epsilon} t e^{l_{cl} t} + \|\mathbold{e}(0))-\overline{\mathbold{e}}(0))\|e^{l_{cl}t}.
\end{multline} 
\end{lemma}
\begin{proof} 
By forward integration, the flows satisfy
    \begin{align*}
       &\| \psi_{\mathbold{e}}(t,\kappa(\mathbold{e}),\mathbold{e}(0)) - \psi_{\overline{\mathbold{e}}}(t,\kappa(\overline{\mathbold{e}})),\overline{\mathbold{e}}(0))\| \\ & \leq\|\mathbold{e}(0))-\overline{\mathbold{e}}(0))\| \! + \!\! \int_{0}^{t}\!\! \big\|f_{cl}(\mathbold{e}) \!-\! f_{cl}(\overline{\mathbold{e}}) \!-\! {d}^{\perp} \hat{\mathbold{n}}_{g_p}^{\perp}\big\| d\tau  
        \\ &\leq \|\mathbold{e}(0))-\overline{\mathbold{e}}(0))\|+\!\! \int_{0}^{t} \!\!l_{cl} \| \mathbold{e}(\tau) - \overline{\mathbold{e}}(\tau)\| d\tau\!  +\! Z^{\perp}_{\epsilon} t.
    \end{align*}
The Gronwall-Bellman Inequality \citep{GB_inequality} yields \eqref{eq:upperbound}.
\end{proof}

\begin{figure}
\centering
    \includegraphics[width=0.99\linewidth]{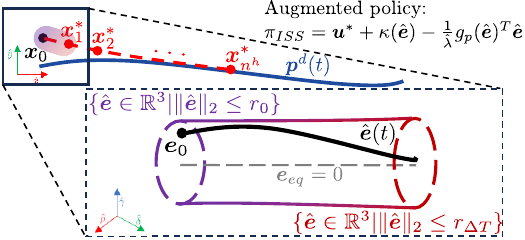}
    \caption{Given reference path $\mathbold{p}^d(t)$ and optimal trajectory $\{\mathbold{x}^*_k, \mathbold{u}^*_k\}_{k=1}^{n_h}$, the maximum tracking error under augmented policy $\pi_{ISS}$ is depicted. By converting the optimal state $\mathbold{x}_1^*$ and current state $\mathbold{x}_0$ into polar coordinate error $\mathbold{e}_0$, the tracking problem reduces to the stabilization of the error $\mathbold{e}_{eq} = 0$. Theorem \ref{thm:main_thm} proves that if $\|\mathbold{e}_0\|_2 \leq r_0$, then the maximum tracking deviation during $t \in [0,\Delta T]$ is $\sup_{t\in [0,\Delta T]} \|\mathbold{e}(t)\|_2$ is upper bounded by $ r_{\Delta T}$.}
    \label{fig:illu}
\end{figure}

Intuitively speaking, the result of Lemma \ref{lemma:2} states that if there is unmatched drift present, then the closed-loop system trajectory between the perturbed and nominal system grows exponentially in time. Building up on Lemma \ref{lemma:1} and Lemma \ref{lemma:2}, our main contribution deduces an \textit{augmented} ancillary controller and the associated maximum discrepancy tube for the perturbed system \eqref{eq:preturbed_tracking_eqn}.

\begin{theorem} \label{thm:main_thm}
    Suppose the system \eqref{eq:polar_eom} is perturbed by input disturbance $\mathbold{d}_{\mathbold{u}}$ and unmatched drift $d_{\perp}$, as in \eqref{eq:preturbed_tracking_eqn}, where the disturbances $\mathbold{d}_{\mathbold{u}}$ and $d_{\perp}$ satisfy Assumption \ref{ass:assump1} given domain $\mathcal{D}^{\mathbold{e}}$.  Given \mbox{$\epsilon \to 0$, let $Z_{\epsilon}$} and $Z_{\epsilon}^{\perp}$ for \mbox{$t\in [0, \Delta T]$} being the discrepancy upper bounds that also satisfy \eqref{eq:desired_bounds_u} and \eqref{eq:desired_bounds_uperp}, respectively. 
    
 Let $l_{V}$ be the Lipschitz constant for the function ${V}$ in Eqn. \eqref{eq:Layp} on the bounded set $\hat{\mathcal{D}}^{\mathbold{e}}$.
Let $\Delta T$ be small enough to satisfy $Z_{\epsilon}^{\perp} \Delta T \exp(l_V\Delta T) \leq \alpha_1$. Consider a time varying radius $r(\tau)$ that satisfies 
\begin{equation} \label{eq:main_result}
  r(\tau) = r_{\tau} \triangleq  \frac{(Z_{\epsilon})^2}{4(\alpha_1 - Z_{\epsilon}^{\perp} \tau e^{l_V \tau})}. 
\end{equation}

Consider the augmented ancillary controller \eqref{eq:kappa_iss} and error dynamics $\dot{\mathbold{e}} \triangleq g_p(\mathbold{e})(\kappa_{ISS}(\mathbold{e}) + \mathbold{d}_{\mathbold{u}})$ and $\dot{\hat{\mathbold{e}}} \triangleq g_p(\hat{\mathbold{e}})(\kappa_{ISS}(\hat{\mathbold{e}}) + \mathbold{d}_{\mathbold{u}}) + {d}^{\perp} \hat{\mathbold{n}}_{g_p}^{\perp}$, where $\mathbold{d}_{\mathbold{u}}(0) = 0$, ${d}_{\perp}(0) = 0$ and $\|\mathbold{e}(0)\| \leq r(0)$. The existence of a $\hat{\lambda}_1 >  0$ such that $-{\alpha}_{3}(\alpha_2^{-1}(V(\hat{\mathbold{e}}))) + \hat{\lambda}_1 V(\hat{\mathbold{e}}) \leq -\hat{\alpha}_3(V(\hat{\mathbold{e}}))$ is true for all $t\in [0,\Delta T]$, where $\hat{\alpha}_3$ is a class $\mathcal{K}_{\infty}$ function, implies that the flow of the closed-loop \footnote{Referring to the error dynamics under augmented control law $\mathbold{u} = \kappa_{ISS}(\hat{e})$ \eqref{eq:kappa_iss} with $\hat{\lambda}_1$ replacing $\lambda_{1}$} tracking error $\hat{\mathbold{e}}(t)$ is bounded by $\|\hat{\mathbold{e}}(t)\| \leq r(t)$  over $t\in [0,\Delta T]$.
\end{theorem}
\begin{proof}
The continuous and differentiable function $\dot{V}$ is bounded over compact set $\hat{\mathcal{D}}^{\mathbold{e}}$ from Assumption \ref{ass:assump1}. Therefore, $V$ is Lipschitz continuous over domain $\hat{\mathcal{D}}^{\mathbold{e}}$.
If there exist $\tilde{\lambda}_1 >0$, then under $\kappa_{ISS}$ yields
    \begin{align*}
       \!\dot{V}(\hat{e}) \!\leq\!  -\alpha_3(\|\hat{\mathbold{e}}\|) \!+\! \frac{\hat{\lambda}_1 Z_{\epsilon}^2}{4} \!+\! |\hat{\mathbold{e}}^{T}\mathbold{d}_{\mathbold{u}}^{\perp}| 
       \!\leq\! -\hat{\alpha}_3(V(\hat{\mathbold{e}})) \!+\! \|\hat{\mathbold{e}}\|_2    Z_{\epsilon}^{\perp}\!.
    \end{align*}
The first inequality follows the proof of Lemma \ref{lem:1} with the addition of $\hat{\mathbold{e}}^{T}\mathbold{d}_u^{\perp} \leq |\hat{\mathbold{e}}^{T}\mathbold{d}_u^{\perp}|$. The second inequality holds from a combination of the existence of $\hat{\lambda}_1$, the definition of $Z_{\epsilon}^{\perp}$~\eqref{eq:desired_bounds_uperp}, and Holder's inequality. 
From Lemma \ref{lemma:2}, we have 
$$
\sup_{\tau \in [0,t]}|V(\mathbold{e}(\tau)) - V(\hat{\mathbold{e}}(\tau))| \leq \left(\sup_{\tau \in [0,t]}  \|\hat{\mathbold{e}}\|_2 \,Z_{\epsilon}^{\perp} \tau \right)e^{l_V\tau },
$$
where $|V(\mathbold{e}(0)) - V(\hat{\mathbold{e}}(0))|= 0$ because $\mathbold{e}(0) = \hat{\mathbold{e}}(0)$.
By adding and subtracting zero and triangle inequalities, we have inequalities $V(\hat{\mathbold{e}}(\tau)) \leq |V(\hat{\mathbold{e}}(\tau))| \leq |V(\mathbold{e}(\tau))| + |V(\hat{\mathbold{e}}(\tau)) - V(\mathbold{e}(\tau))|$.  Recall the robust control invariance set \eqref{eq:omega_iss}. Let the initial error radius $r(0)$ take the explicit value
\begin{equation} \label{eq:r0_definition}
    r(0) = \frac{(Z_{\epsilon})^2}{4\alpha_{1}^{-1} }.
\end{equation}
Further, since $\|\mathbold{e}(0)\| \leq r(0)$ and there exists a $\hat{\lambda}_1 >  0$, the unperturbed error dynamics satisfies $\|\mathbold{e}(t))\| \leq r(0)$ for all $t\in [0,\Delta T]$, by Lemma \ref{lemma:1}.  Combining these inequalities, we can show
\begin{align} \label{eq:r_dt_definition}
\sup_{\tau\in [0,t]}\!\! V(\hat{\mathbold{e}}(\tau)) \!\leq\! \alpha_1 r_0 \!+\!\left( \sup_{t\in [0,\Delta T]}\! \|\hat{\mathbold{e}}(t)\|Z^{\perp}_\epsilon \Delta T\right)\! e^{l_V \Delta T}.
\end{align}
because $V(\hat{\mathbold{e}}) \geq \alpha_1\|\hat{\mathbold{e}}\|$. Substituting this inequality into equation \eqref{eq:r_dt_definition}, we arrive at
\begin{equation} \label{eq:tb_combined}
   \!\!\!\sup_{\tau\in [0,t]} \!\|\hat{\mathbold{e}}(\tau)\| \!\leq \! r_0 \!+\!\! \sup_{\tau \in [0,t]}\!\! \|\hat{\mathbold{e}}(\tau)\|\frac{Z_{\epsilon}^{\perp}\tau  e^{l_V \tau}}{\alpha_1}.
\end{equation}
Lastly, combining equations \eqref{eq:tb_combined} and \eqref{eq:r0_definition} 
 leads to the desired maximum tracking error bound \eqref{eq:main_result}. 
\end{proof}
The increasing tube (radius-wise) defined by time-dependent radius $r(t)$ \eqref{eq:main_result} in the interval $t\in[0,\Delta T]$ is the maximum tracking error tube, illustrated in Figure \ref{fig:illu}. Note, every planned desired trajectory $\mathbold{x}^*(t)$ resulting from desired inputs $\mathbold{u}^*(t)$ starting at state $\mathbold{x}_0$ is equivalent to stabilizing the tracking error from $\mathbold{e}_0$ to $[0,0,0]^T$. Leveraging the result of Theorem~\ref{thm:main_thm}, the perturbed model tracking error $e(t)$ within each planning interval $t\in[0,\Delta T]$ will be inside a growing tube along the planned trajectory $\mathbold{x}^*(t)$ for $t\in [0,\Delta T]$ under the augmented controller $\kappa_{ISS}$, if the initial tracking error $\mathbold{e}_0$ is inside a ball of radius $r_0$ $\eqref{eq:r0_definition}$, as illustrated in Fig. \ref{fig:illu}. 

\begin{rem}
The condition $\epsilon \to 0$ can require $Z_{\epsilon}, Z_{\epsilon}^{\perp} \to \infty$ which is unpractical for robotic applications. Therefore for $\epsilon \in (0,1)$, one can translate the result of Theorem  \ref{thm:main_thm} as the inequality $\eqref{eq:main_result}$ is true for $(1-\epsilon)$ confidence level. 
\end{rem}



\section{Discrepancy-Aware Planning} \label{sec:dA_MPPI}

Because differential-drive ground vehicles are underactuated, the unmatched uncertainties cannot be compensated in a controller synthesize problem. Nevertheless, the closed-loop trajectory drift (Lemma \ref{lemma:2}) can be effectively modulated through trajectory re-planning such as through heading adjustments, as demonstrated by the tube-MPC method \citep{bastos2021energy}. 

To complete our framework, this section proposes a novel Discrepancy-Aware Planning algorithm that optimizes reference tracking while providing safety constraint satisfaction at a desired risk tolerance. We focus on two challenges: integrating maximum tracking error into the occupancy map to infer probabilistic safe traversal and finding optimal trajectories and inputs for the policy \eqref{eq:policy_over_arching}.


\begin{figure}
\centering
    \includegraphics[width=0.99\linewidth]{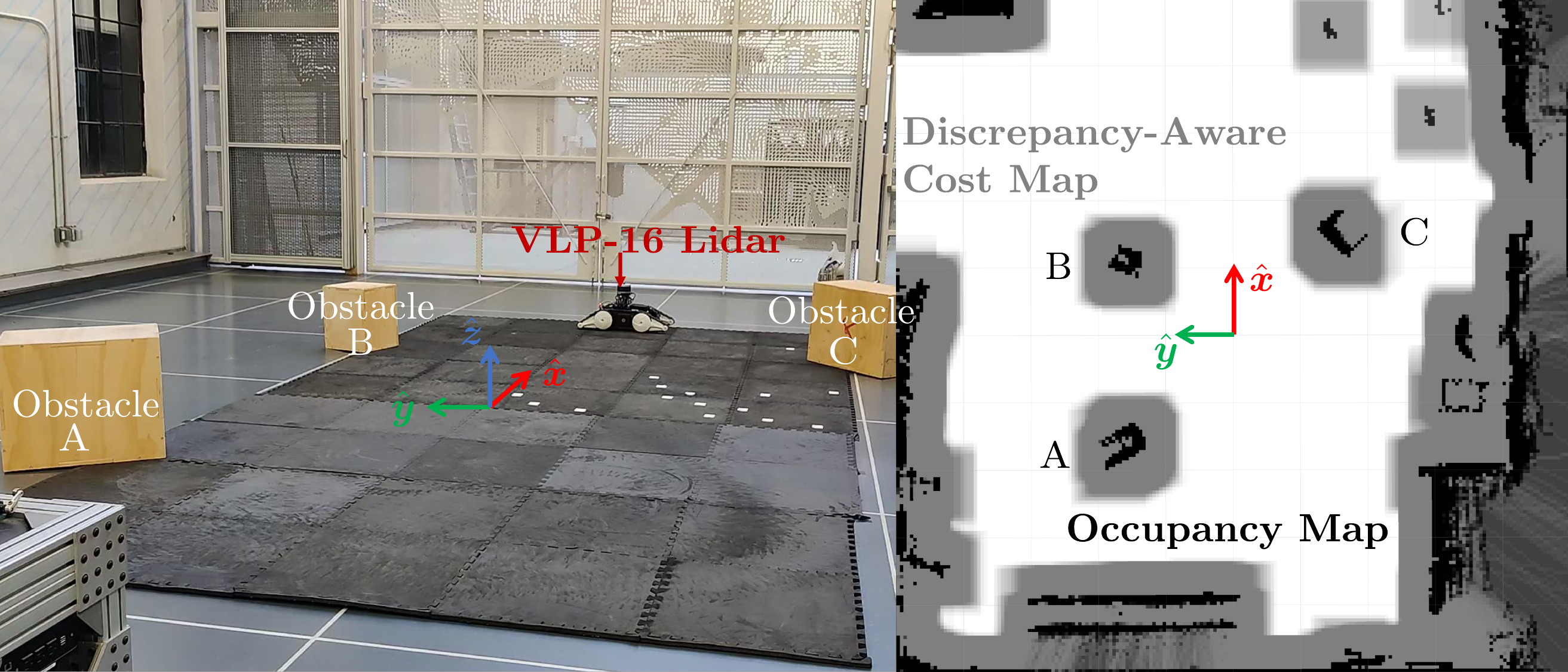}
    \caption{(Left) Configuration B test set up capture where three stationary obstacles are placed such that the vehicle needs to deviate from its reference trajectory to avoid. (Right) An overlay of the raw occupancy map and the discrepancy-aware cost map $\mathcal{C}_{\epsilon}$ with $\epsilon \in (0.0013,0.0016)$ resulting a $N_{\epsilon}= 16$ composed by $r_{ego} = 0.40~\mbox{m}$ and $r_{buff} = 0.15~\mbox{m}$ with $r_{map} = 0.05~\mbox{m}$.}
    \label{fig:maps}
\end{figure}

\subsection{Probabilistic Safety-Critical Planning}
Given a desired trajectory, the optimal tracking input can be found using a deterministic FTOCP with the cost:
\begin{equation} \label{eq:cost_standard}
   \! \mathcal{L}_{track} = \!\mathcal{L}_{T}(\delta\mathbold{p}_{t+n_h})\! + \!\sum_{k=0}^{n^{h}-1}\left( \mathcal{L}_s(\delta\mathbold{p}_k) + \frac{1}{2}\mathbold{u}_k^{T}R\mathbold{u}_k\right)
\end{equation}
where $\delta\mathbold{p}_{k} = \mathbold{p}_k - \mathbold{p}^{d}_k$ is the output (position) tracking error in a fixed frame. Functions $\mathcal{L}_{T} \triangleq \delta\mathbold{p}_{t+{n}^h}^{T}Q_{T}\delta\mathbold{p}_{t+{n}^h}$ and $\mathcal{L}_{s} \triangleq \delta\mathbold{p}_{k}^{T}Q\delta\mathbold{p}_{k}$ for $k\in \{0,\hdots,n_h-1\}$ are convex terminal and stage costs, respectively. Matrices $Q_T$, $Q$, $R\in \mathbb{R}^{2\times 2}$ are positive definite. Note, we apply a maximum tracking cost threshold, $\overline{\mathcal{L}}_{track} \in \mathbb{R}_{>0}$ \footnote{The parameter $\overline{\mathcal{L}}_{track}$ stops the cost-to-go from becoming unbounded as the output tracking error increases which needs to be tuned. The authors infer $\overline{\mathcal{L}}_{track}$ from the system's input limits and the planning horizon such that the maximum stage cost corresponding to the physical limits imposed by the robotic hardware.}, to the cost as $\mathcal{L}_{track} = \min(\mathcal{L}_{track}, \overline{\mathcal{L}}_{track})$. 

Given a vehicle operating state $X_t$ and its projected nominal state $\mathbold{x}_t$, the solution to the following discrete-time optimization problem yields the desired probabilistically safe input sequence
\begin{subequations} \label{eq:FTOCP}
\begin{align}
      \{ \mathbold{u}^{*}_{t+k}\}&_{k=0}^{n_h - 1}  =\!\!\!\!\min_{\tiny \begin{array}{c} 
           \{\mathbold{u}_{t+k}\}^{n_h -1}_{k=0} 
        \end{array}}\normalsize \!\!\! \sum_{k=0}^{n_h - 1} \mathcal{L}_{track} \\
     \mbox{s.t.} \quad       &X_{t+k+1} = \hat{f}^d_{t+k} (X_{t+k})
     \label{eq:discrete_time_full}\\
    & \mathbold{x}_{t+k+1} = P_{X\to \mathbold{x}}(X_{t+k+1}) \in \mathcal{D}^{\mathbold{u}}\label{eq:optimization_stateninput_constraint}\\
     & \mathbold{u}_{t+k} = P_{X\to \mathbold{u}}(X_{t+k}) \in \mathcal{D}^{\mathbold{u}}\label{eq:optimization_stateninput_constraint}    
     , \quad \forall k \in\mathbb{Z}_{0}^{n_h-1} \\
      &\mathbold{x}_{t} = P_{X\to \mathbold{x}}({\mathbold{x}}_{t}) \label{eq:initial_condition}\\
     &\mathbb{P}_{\Omega}[ M\mathbold{x}_{t+1} \in \overline{\mathcal{O}}] \leq \epsilon ,\label{eq:optimization_obsavoid_setconstraint}
\end{align}
\end{subequations}
 that avoids all occupied grids in the map $\mathcal{O}$ at a minimum probability $(1-\epsilon)$ percent. Equation \eqref{eq:discrete_time_full} represents a discrete-time update of the true system  operating state $X$ with time discretization step size $\Delta T$. The matrix $M\in \mathbb{R}^{2\times 3}$ maps from nominal states $\mathbold{x}$ to vehicle positions $\mathbold{p}$, and constraint \eqref{eq:initial_condition} means the initial condition is set by the current state.
 Inequality \eqref{eq:optimization_obsavoid_setconstraint} constrains the probability of the planned vehicle's position, at time $t+1$ given the current nominal state $\mathbold{x}_t$, to be inside an augmented occupied set, denoted as $\overline{\mathcal{O}} \subseteq \mathbb{R}^2$, is less than $\epsilon$. The set $\overline{\mathcal{O}}$ is constructed from $\mathcal{O}$ with the addition of the vehicle geometry such that if the point mass at position $\mathbold{p} \in \mathbb{R}^2$ satisfied $\mathbold{p} \notin \overline{\mathcal{O}}$, then the vehicle occupied points given center position $\mathbold{p}$ denoted as the set $\mathcal{S}_{\mathbold{p}} \subset \mathbb{R}^2$ satisfied $\mathcal{S}_{\mathbold{p}} \cap \mathcal{O} = \emptyset$. Physically speaking, the vehicle should not enter the obstacle-occupied grids given by the map $\mathcal{O}$ more frequently than the specified risk tolerance. 

The optimization problem \eqref{eq:FTOCP} is difficult to solve because the true system dynamics \eqref{eq:discrete_time_full}  is unknown and unobservable, constraint \eqref{eq:optimization_obsavoid_setconstraint} is nondeterministic, and the occupancy percentage given map $\mathcal{O}$ given continuous position $x,y$ is discontinuous. As a solution, we proposed to reformulate \eqref{eq:FTOCP} into a deterministic program while providing the same or lower obstacle collision tolerance $\epsilon$. 

Recall that polar coordinates state $\rho$ represents the displacement between the vehicle's actual position and the planned one. From Theorem~\ref{thm:main_thm}, we have
\begin{equation*}
    \sup_{t\in[0,\Delta]}|\hat{\rho}(t)| \!= \!\! \sup_{t\in[0,\Delta T]}\!\!\!|\mathbold{i}^{T}\hat{\mathbold{e}}(t)| \leq \|\mathbold{i}\| \!\!\sup_{t\in[0,\Delta T]} \!\!\!\|\hat{\mathbold{e}}(t)\|\!\leq\! r_{\Delta T},
\end{equation*}
where $\mathbold{i} = [1,0,0]^{T}$. Given the current nominal state $\mathbold{x}_t$, if the initial tracking error arising from the planned state $\mathbold{x}_{t+1}^*$  satisfies $\|\mathbold{e}_{0}\| \leq r_0$, we can then guarantee up to $(1 - \epsilon)$ confidence that vehicle position $\mathbold{p}_{t+1}$, at time $t+\Delta t$, is an element of the following set
\begin{equation} \label{eq:set_ISS}
    \mathcal{S}_{ISS} =  \{\mathbold{p}\in \mathbb{R}^2 \,|\, \|\mathbold{p}_{t+1}(\mathbold{x}_t,\kappa(\mathbold{e}_0)) - \mathbold{p}^*_{t+1}\| \leq r_{\Delta T}\}.
\end{equation}
That is, if the constraint  $\|\mathbold{e}_{0}(\mathbold{x}_t,\mathbold{x}_{t+1}^*)\| \leq r_0$ is satisfied, then enforcing an added distance gap of $r_{\Delta t}$ between the ego vehicle and the obstacle is sufficient to guarantee obstacle avoidance with a $(1-\epsilon)$ confidence.

\begin{theorem} \label{thm:main_thm2}
Consider a differential-driven UGV with nominal dynamics \eqref{eq:origin}, true dynamics \eqref{eq:true}, and reference trajectory $\mathbold{p}^{d}(t)$.
Denote the current vehicle operating state as $X_t$ and its projected nominal state as $\mathbold{x}_t$. The polar tracking error satisfies nominal dynamics \eqref{eq:polar_eom}. Consider the tracking control policy \eqref{eq:policy_over_arching} where the nominal ancillary controller $\kappa$ \eqref{eq:fb_law} is used for stabilizing the tracking error $\mathbold{e}$. 

    Suppose \eqref{eq:polar_eom} is perturbed by an additive matched disturbance, $g_{p}(\mathbold{e})\mathbold{d}_{\mathbold{u}}$, and an additive unmatched model disturbance $d_{\perp}\hat{\mathbold{n}}_{g_p}^{\perp}$ to \eqref{eq:preturbed_tracking_eqn}, and there exist $Z_{\epsilon}\in \mathbb{R}_{>0}$ satisfying \eqref{eq:desired_bounds_u} and $Z_{\epsilon}^{\perp}\in \mathbb{R}_{>0}$ satisfying \eqref{eq:desired_bounds_uperp} given a risk tolerance $\epsilon \in [0,1]$. Suppose there exists $\tilde{\lambda}_1\in\mathbb{R}_{>0}$ in Theorem \ref{thm:main_thm}. Let $\{\mathcal{FS}^*_{t+k}\}_{k=0}^{n_h-1}$ be the set of admissible inputs \footnote{ If a input sequences $\{\mathbold{u}_{t+k}\}_{k=0}^{n_h-1}$ satisfies the constraints \eqref{eq:discrete_time_nominal} - \eqref{eq:obs_input_aug} given current state  $\mathbold{x}_t$ then such input sequence is an element of the set $\{\mathcal{FS}^{*}_{t+k}\}_{k=0}^{n_h-1}$, $\mathbold{u}_{t+k}\in \mathcal{FS}^{*}_{t+k}$ for all $k\in \mathbb{Z}_{k=0}^{n_h-1}$.} for the following optimization problem.
\begin{subequations} \label{eq:FTOCP_mine}
\begin{align}
      \{ \mathbold{u}^{*}_{t+k}\}&_{k=0}^{n_h - 1}  =\!\!\!\!\min_{\tiny \begin{array}{c} 
           \{\mathbold{u}_{t+k}\}^{n_h - 1}_{k=0} 
        \end{array}}\normalsize \!\!\! \sum_{k=0}^{n_h - 1} \mathcal{L}_{track} \\
      \mbox{s.t.}\quad &      \mathbold{x}_{t+k+1} = f^d_{t+k}(\mathbold{x}_{t+k},\mathbold{u}_{t+k}, \Delta t) \label{eq:discrete_time_nominal}\\
    &  \mathbold{e}_{k} = f_{\mathbold{e}}(\mathbold{x}_{t+k}, \mathbold{x}_{t+k+1}) \\
   & \|\mathbold{e}_{k}\| \leq r_{0} \label{eq:obs_avoid_init}\\
 & {C}(\mathbold{p}_{t+k+1},r_{(k+1)\Delta T}) \cap \overline{\mathcal{O}} = \emptyset \label{eq:obs_avoid_desired}\\
     & \mathbold{u}_{t+k} + \kappa_{ISS}(\mathbold{e}_{k})\in \mathcal{D}^{\mathbold{u}},\quad \forall k\in \mathbb{Z}_{0}^{n_h-1} \label{eq:obs_input_aug}\\
     &\mathbold{x}_t = {\mathbold{x}}_{t}, 
\end{align}
\end{subequations}
where equality constraint \eqref{eq:discrete_time_nominal} is the discrete-time version of nominal model \eqref{eq:origin}. The set ${C}(\mathbold{p}_{k+1},r_{\Delta T})$ in \eqref{eq:obs_avoid_desired} denotes the closed disc with center $\mathbold{p}_{k+1} \in \mathbb{R}^2$ and radius $r_{\Delta T(k+1)}\in \mathbb{R}_{>0}$.
   Let $\{\mathcal{FS}_k\}_{k=0}^{n_h-1}$ be the sequence of feasible input sets to the optimization problem \eqref{eq:FTOCP}. If $r_{t}$ is computed using \eqref{eq:main_result}, then $\mathcal{FS}^*_{t+k} \subseteq \mathcal{FS}_{t+k}$, for all $k\in \mathbb{Z}_{0}^{n_h-1}$.
\end{theorem}
\begin{proof}
    As an overview, we show set inclusion using a standard approach by showing every element of the set $\mathcal{FS}^{*}_{k}$ also belongs to the set $\mathcal{FS}_k$ for all $k \in \mathbb{Z}_{0}^{n_h-1}$. Consider $k = 0$. From Theorem \ref{thm:main_thm}, under the augmented ancillary controller $\kappa_{ISS}$, we can assert that the position flow of the perturbed system $\mathbold{\psi}_{\mathbold{p}}(\tau)$ for $\tau \in [t,t+\Delta T]$ under any input $\mathbold{u}_t \in \mathcal{FS}^*_t+k$  with initial position $\mathbold{p}_t$ and $(1-\epsilon)$ confidence, satisfies the following, 
    \begin{align*}
        \mathbold{\psi}_{\mathbold{p}}(\tau)\! -\! \mathbold{p}^*(\tau) \leq \!\!\!\!  \sup_{\tau \in [t,t+\Delta T]} \!\!\|\delta \mathbold{p}(\tau)\| \leq \!\!\!\!\sup_{\tau \in [t,t+\Delta T]}\!\!|\rho(\tau)| \leq r_{\Delta T}.
    \end{align*}
    The above inequality implies, for all $\mathbold{u}_t\in \mathcal{FS}^*_0$, satisfaction of the inequality constraints \eqref{eq:obs_avoid_init} and \eqref{eq:optimization_stateninput_constraint} is equivalent to the statement that the perturbed system's position flow $\mathbold{\psi}_{\mathbold{p}}(\tau)$ for $\tau \in [t, t + \Delta T]$ does not intersect an occupied grid with $(1-\epsilon)$ confidence. Therefore, control inputs $\mathbold{u}_t\in \mathcal{FS}_{t}^*$ are also feasible solutions to optimization problem \eqref{eq:FTOCP} yielding $\mathbold{u}_t \in \mathcal{FS}_t$, meaning $\mathcal{FS}^*_0 \subseteq \mathcal{FS}_0$. For $k = 1$, we can similarly inspect the flow generated by the composite input
    \begin{align*}
        \mathbold{u}(\tau) = \begin{cases}
            \mathbold{u}_t \in \mathcal{FS}^*_{t}, \quad &\tau \in [t, t+\Delta T] \\
            \mathbold{u}_{t+1}\in \mathcal{FS}^*_{t+1}, \quad &\tau \in [t + \Delta T, t+2\Delta T].
        \end{cases}
    \end{align*}
Building on the result for $k=0$, we have $\mathbold{\psi}_{\mathbold{p}}(\tau)\leq \mathbold{p}^*(\tau)  + r_{2\Delta T}$ for $\tau \in [t, t+2\Delta T]$ with confidence $1-\epsilon$. Therefore, $\forall \mathbold{u}_{t+1} \in \mathcal{FS}^*_{t+1}$ are also member of set $\mathcal{FS}_{t+1}$, yielding $\mathcal{FS}^*_{t+1} \subseteq \mathcal{FS}_{t+1}$. The proof for the remaining $k\in \mathbb{Z}_{k=2}^{n_h-1}$ follows by induction. 
\end{proof}

Based on Theorem \ref{thm:main_thm2}, we now can solve the deterministic optimization problem \eqref{eq:FTOCP_mine} with an inaccurate but known nominal model so that it guarantees the desired obstacle avoidance behavior and satisfies the original optimization problem \eqref{eq:FTOCP}. Lastly, we assume that the optimization-based planner receives measurement updates at every $\Delta T$ seconds. Following standard receding horizon MPC implementation, the planner replans at the same rates as measurement updates with constraint \eqref{eq:obs_avoid_init} only applied at $k=0$ \footnote{Enforcing trajectory level safety, by adding safety distance of $r_{(k+1)\Delta T}$ at future horizon steps, $k \in \mathbb{Z}_{1}^{n_h-1}$, can lead to overly conservative behavior.}. 

\begin{rem}
    In Theorem \ref{thm:main_thm2}, a circular robot shape \eqref{eq:obs_avoid_desired} can be replaced by a tighter polytopic boundary. In this case, the cost map introduced in the following section may become orientation-dependent. 
\end{rem}


\begin{algorithm}[!h]
\caption{Discrepancy-Aware MPPI}
\small
\label{alg:2} 
\KwData{Map parameters $N_{\epsilon}$, $r_{map}$, $x_{\mathcal{O},c}$, $y_{\mathcal{O},c}$, occupancy map $\mathcal{O}$, current nominal state $\mathbold{x}_{t}$, goal position $\mathbold{p}^{d}_{t}$, cost function parameters: $Q$, $Q_T$, $R$, $\alpha_{cost}$, $\alpha_{shift}$, MPPI parameters $\Sigma_{\mathbold{u}}$, $N_{sample}$,$\lambda$,$\alpha_{ISS}$, initial control sequence $\{\overline{\mathbold{u}}\}_{0}^{n_h-1}$, and horizon $n_h$.}

\KwResult{$\{\mathbold{x}^{*}_{i}\}_{i=1}^{n_h}$, $\{\mathbold{u}^*_{i}\}_{i=0}^{n_h-1}$, $\mathbold{u}_{send}$}

Create $l_{map} + 2N_{\epsilon}$ by $w_{map} + 2N_{\epsilon}$ grid map $\hat{\mathcal{C}}_{buffer}$.

Create $l_{map} + 2N_{\epsilon}$ by $w_{map} + 2N_{\epsilon}$ enlarged occupancy map $\hat{\mathcal{O}}$ based on $\mathcal{O}$.

\While{ task not completed}{
\For{ $i = - N_{\epsilon}, i < N_{\epsilon}, i ++$}{
\For{$j = - N_{\epsilon}, j < N_{\epsilon}, j ++$}{
    $C_{buffer}(N_{\epsilon} + i:N_{\epsilon} + l_{map} + i - 1, N_{\epsilon} + j:N_{\epsilon} + w_{map} + j - 1) += \frac{\alpha_{shift}\hat{\mathcal{O}}_{i,j}^{s}}{100(\sqrt{i^2 + j^2 + 1})}$.
}
}
$\mathcal{C}_{\epsilon} = \lceil\mathcal{C}_{buffer}(N_{\epsilon}:N_{\epsilon}+l_{map}-1, N_{\epsilon}:N_{\epsilon}+w_{map}-1),\overline{\mathcal{L}}_{track}\rceil$.

\For{ $k = 0, k < N_{sample} - 1, k++$}{
Draw $\delta \mathbold{u}$ from $\mathcal{N}(\mathbold{0}, \Sigma_{\mathbold{u}})$ $n_h$ times.
\For{$i = 0, i<n_h-1, i++$}{
$\mathbold{x}_{i+1} = \mathbold{x}_{i} +g(\mathbold{x}_i)(\min(\max(\overline{\mathbold{u}}_i + \delta \mathbold{u}_{i},\mathbold{u}_{min}),\mathbold{u}_{max})\Delta t$
}
Evaluate the $k^{th}$ MPPI trajectory cost $\overline{\mathcal{L}}_k$ using \eqref{eq:augmented_cost}.
}


Compute the optimal MPPI input sequence $\{\mathbold{u}_{i}^*\}_{i=0}^{n_h -1 }$ using equations \eqref{eq:is_1}-\eqref{eq:is_2} with $\overline{\mathcal{L}}_k$ as the flow costs. 

Compute best sampled state trajectory $\{\mathbold{x}_i^{*}\}_{i=1}^{n_h}$ and use $\mathbold{x}_{1}^{*}$ and $\mathbold{x}_0$ to compute polar tracking error $\hat{\mathbold{e}}$. 

Use \eqref{eq:ucmd} to compute the input that will be sent to actuators. 

For $i\!=\!0:n_h-2$: $\overline{\mathbold{u}}_{i} = \mathbold{u}^*_{i+1}$. Fill $\overline{\mathbold{u}}^*_{n_h-1}$ using flatness relationship \eqref{eq:flat_input_v}.
}
\normalsize
\end{algorithm}

\subsection{Discrepancy-Aware MPPI}
The nonlinear program \eqref{eq:FTOCP_mine} can still be numerically challenging to solve, especially when set $\mathcal{O}$ is represented in a grid map. Therefore, we propose a \textit{discrepancy-aware cost map} to encode the obstacle avoidance constraint \eqref{eq:obs_avoid_desired}. This map facilitates cost minimization, ensuring robust obstacle avoidance amid model inaccuracy and vehicle geometry.
Consider the following augmented cost function which combines the reference tracking cost \eqref{eq:cost_standard} with a collision penalty,
\begin{equation} \label{eq:augmented_cost}
    \overline{\mathcal{L}} = \mathcal{L}_{track} + \sum_{k=1}^{n_h} \mathcal{L}_{C}(\mathbold{p}_k, \mathcal{C}_{\epsilon}) + \alpha_{ISS}\mathbold{1}_{\|\mathbold{p}_{t+1} - \mathbold{p}_{t}\|\geq r_{0}},
\end{equation}
where cost function $\mathcal{L}_{C}:\mathbb{R}^2 \times \mathbb{R}^{l_{map} \times w_{map}} \to \mathbb{R}$ takes in a position $\mathbold{p} \in \mathbb{R}^2$ and a $l_{map} \times w_{map}$ sized discrepancy aware cost map to produce a collision cost. The parameter $\alpha_{ISS} \in \mathbb{R}_{>0}$ is chosen large enough to force the planner to produce a trajectory that satisfies the requirement $\|\mathbold{e}_0\| \leq r_0$ \footnote{The term $\alpha_{ISS}\mathbold{1}_{\|\mathbold{p}_{t+1} - \mathbold{p}_t\|\geq r_{0}}$ can be dropped if sample trimming is incorporated where a MPPI sampled trajectory is omitted when the inequality $\|\mathbold{p}_{t+1} - \mathbold{p}_{t}\|\geq r_{0}$ holds. The authors used a replanning scheme instead of sample trimming in the experimental validation section where if the inequality $\|\mathbold{p}_1 - \mathbold{p}_0\|\geq r_{0}$ holds, a new set of input discrepancies are sampled until the converse is true. }.

Using the discrete-time nominal model \eqref{eq:discrete_time_nominal}, we apply MPPI to derive sub-optimal trajectory and input sequences ($\{\mathbold{x}^*_{t+k}, \mathbold{u}^*_{t+k}\}_{k=1}^{n_h}$) of optimization problem \eqref{eq:FTOCP_mine} using the importance sampling law \eqref{eq:is_1}-\eqref{eq:is_2} and the proposed discrepancy-aware cost function \eqref{eq:augmented_cost}. To manage input constraints, we introduce an element-wise clamping function $\mathbold{u}_{clamp} = \max(\min(\mathbold{u}_{max},\mathbold{u}),\mathbold{u}_{min})$ \citep{williams2018robust}, which does not affect the MPPI algorithm's convergence. The initial input sequence (warm start) is obtained using the nominal dynamics and flatness properties as \eqref{eq:flat_input_v}. Algorithm \ref{alg:2} describes this process, including the construction of the cost map $\mathcal{C}_{\epsilon}$.

\subsection{Discrepancy-Aware Cost Map}

Given that the nominal model \eqref{eq:discrete_time_nominal} focuses on the UGVs' center of mass, we account for the vehicle footprint via the smallest circumscribing circle with radius, $r_{ego}$, see Fig. \ref{fig:cost_map_build}. We then enlarge each occupied grid by $r_{ego}$. Additionally, to accommodate model mismatches, we further enlarge the grids by $r_{\Delta T}$, creating a ``collision buffer" around each occupied grid. To encode this collision buffer in the nominal grid map, we define $N_{\epsilon} \triangleq \lceil (r_{\Delta T} + r_{ego})/r_{map} \rceil$, a positive integer that denotes the buffer's grid size. The discrepancy-aware cost map can be interpreted as a $N_{\epsilon}$ grid ``inflation" of the occupancy map to add a safety buffer between the ego vehicle and surrounding obstacles (see Fig. \ref{fig:cost_map_build} for an example of such inflation to account for the unmodeled disturbances). The cost map construction is described in Algorithm \ref{alg:2}, with added details in \ref{Appendix:1}. 

The discrepancy-aware cost map, denoted as $\mathcal{C}_{\epsilon}$, satisfies the following property:
\begin{equation}
    \mathcal{C}{\epsilon}(\mathbold{p}) \geq \overline{\mathcal{L}}_{track} \quad  \mbox{iff}\quad  {C}(\mathbold{p},r_{\Delta T}+r_{ego}) \cap \overline{\mathcal{O}} \neq \emptyset 
\end{equation}
for all positions $\mathbold{p}\in \mathbb{R}^2$ that place in the occupancy map. 

\label{sec:experiment}
\begin{figure}
\centering
    \includegraphics[width=0.99\linewidth]{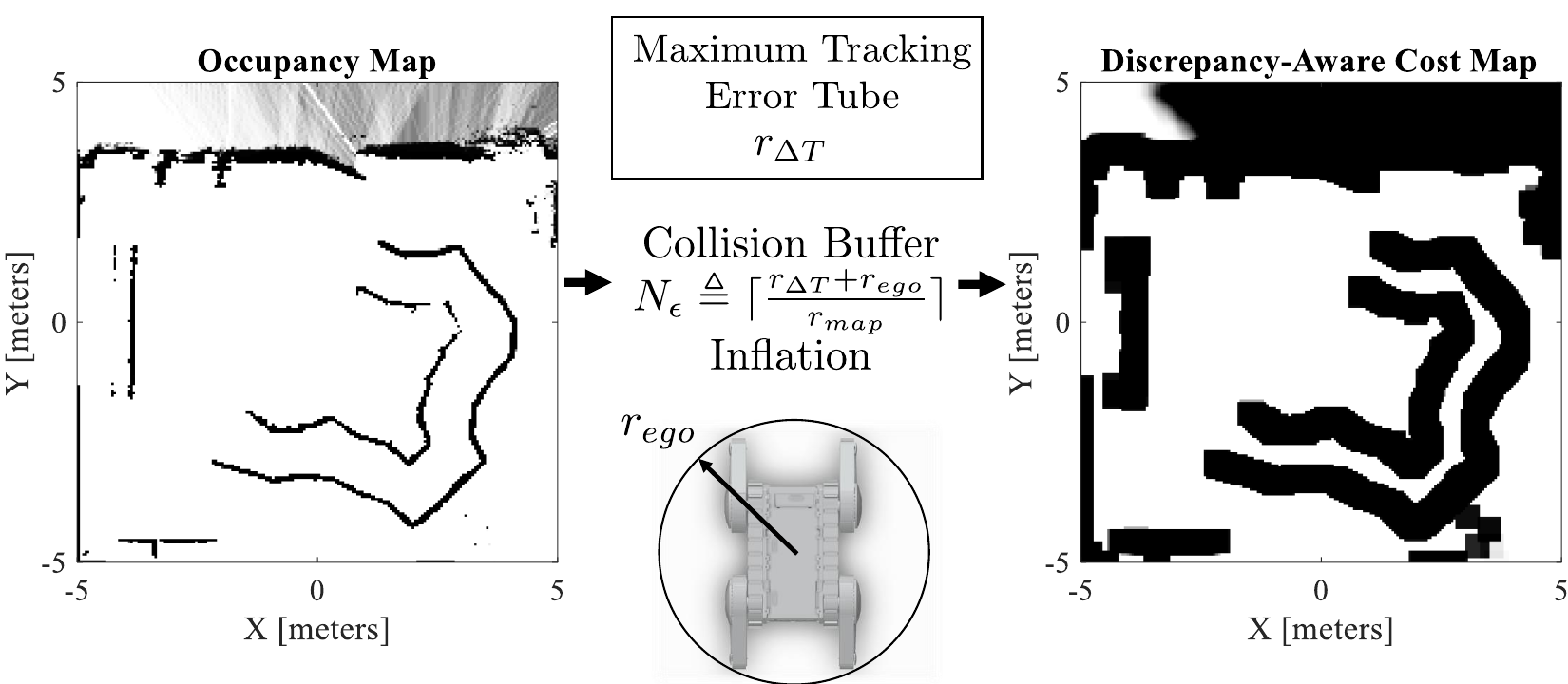}
    \caption{A pictorial overview of the construction of the Discrepancy-Aware Cost Map from occupancy map and collision buffer.}
    \label{fig:cost_map_build}
\end{figure}

For MPPI cost evaluation, we map the position $\mathbold{p} = [x, y]^T$ to the corresponding grid indices on $\mathcal{C}{\epsilon}$, calculating $\mathcal{L}_{C}(\mathbold{p}, \mathcal{C}{\epsilon})$ as the cost map entry at those grid indices \footnote{
Explicitly, we compute $X_{\mathcal{C}} = \lfloor \frac{x}{r_{map}} + \frac{l_{map}}{2}\rfloor$ and  $Y_{\mathcal{C}}  = \lfloor \frac{y}{r_{map}} + \frac{w_{map}}{2}\rfloor$ given $\mathbold{p}_k$. Then, $\mathcal{L}_{C}(\mathbold{p}_{k},\mathcal{C}_{\epsilon})$ is the $(X_\mathcal{C},Y_\mathcal{C})^{th}$ entries of $\mathcal{C}_{\epsilon}$.}.

As an extension to Theorems \ref{thm:main_thm} and \ref{thm:main_thm2}, we have the following guarantees of the proposed data-driven planning and control framework. 
\begin{corollary}
    Suppose sufficient training samples, ${S}$, of the form \eqref{eq:tuple_training}, are available
 to calculate $Z_{\epsilon}, Z_{\epsilon}^{\perp}$ using Algorithm \ref{alg:1} and a user-defined risk-level $\epsilon \in (0,1)$. Suppose there exists $\tilde{\lambda}_1 >0$ such that Theorem \ref{thm:main_thm} holds. Under the same conditions and assumptions as Theorem \ref{thm:main_thm2}, a local minimum trajectory $\{\mathbold{x}^*_{t+k}\}_{k=1}^{n_h}$ given by input sequence $\{\mathbold{u}^*_{t+k}\}_{k=0}^{n_h-1}$ is obtained using the discrepancy-aware MPPI algorithm, given by Algorithm \ref{alg:2}. Applying control input 
 \begin{eqnarray} \label{eq:final_input}
     \mathbold{u}_t = \mathbold{u}_t^* + \kappa_{ISS}(\mathbold{e}_0)
 \end{eqnarray}
 is sufficient to avoid the occupied grids in map $\mathcal{O}$ with a minimum $(1-\epsilon)$ confidence, if the following conditions are satisfied:
 \begin{itemize}
     \item The control inputs $\mathbold{u}_t$ from Eqn. \eqref{eq:final_input} satisfy input constraint, i.e.  $\mathbold{u}_t \in \mathcal{D}_{\mathbold{u}}$.
     \item The trajectory cost $\overline{\mathcal{L}}(\{\mathbold{x}^*_k,\mathbold{u}^*_k\}_{k=0}^{n_h-1}) < \overline{\mathcal{L}}_{track}$. 
\end{itemize}\label{coro:1}
\end{corollary} 
\begin{proof}
   Based on the cost assignment of $\mathcal{C}_{\epsilon}$, if the total cost of the planned trajectory $\overline{\mathcal{L}}(\{\mathbold{x}^*_k,\mathbold{u}^*_k\}_{k=0}^{n_h-1})$ is less than $\overline{\mathcal{L}}_{track}$, we can automatically guarantee that $\mathcal{L}_{C}(\mathbold{p}^*_{t+1},\mathcal{C}_\epsilon) \leq \overline{\mathcal{L}}_{track}$, i.e. collision safety with confidence $1-\epsilon$. 
   Given the control policy \eqref{eq:final_input} also satisfies the control input limit, we can conclude that the flow of the perturbed system \eqref{eq:preturbed_tracking_eqn} is within an expanding maximum track error tube of radius $r_{t}$, calculated using \eqref{eq:main_result}, along the planned position trajectory $\{\mathbold{p}_{t+k}^*\}_{k=0}^{1}$ with $(1-\epsilon)$ confidence.
\end{proof}

The theoretical guarantees of Corollary \ref{coro:1} serve as safety validation and verification of the plan and also a safe controller to be synthesized, as shown in the driver-assist application in Section \ref{sec:experiment}.

\section{Trajectory Tracking with Unknown Obstacles}
\label{sec:experiment}
\begin{figure}
\centering
    \includegraphics[width=0.99\linewidth]{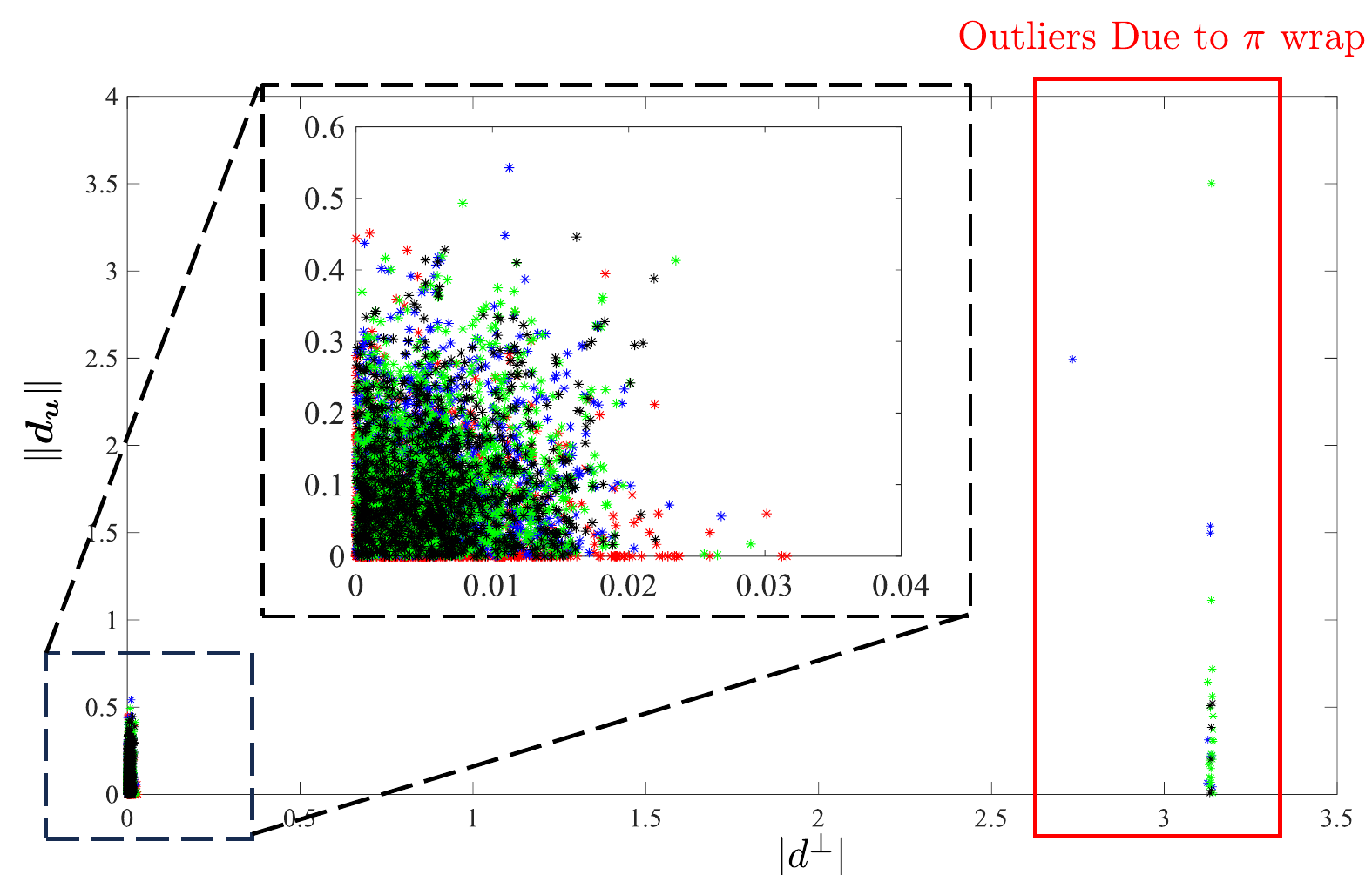}
    \caption{Data-driven discrepancy identification results from configuration B with $T_{lap} = \{20,30,40,50\}$ seconds. From the training set, we plot the absolute unmatched drift against the input discrepancy norm, observing non-negligible model discrepancy upper bounds. Outliers are observed due to incorrect $\pi$ wrapping which can skew the empirical distribution formed by the nonconformity scores as well as the conformal-driven discrepancies. }
    \label{fig:training_results}
\end{figure}


We first validated the proposed framework through high-speed trajectory tracking experiments in the presence of stationary obstacles. These experiments are conducted using four ground and vehicle configurations, as depicted in Fig.~\ref{fig:schematics}.

The UGV is a Flipper Pro by Rover Robotics, with $r_{ego}=0.39~\mbox{m}$ in the flipper down configuration and $r_{ego} = 0.3~\mbox{m}$ in the flipper up configuration. There are three stationary boxes in the test field acting as obstacles, strategically placed to test the robot's obstacle avoidance capabilities (refer to Fig.~\ref{fig:maps} for layout details).

We construct a 2D occupancy map centered at $(0,0)$ with a $0.05~\mbox{m}$ grid size, updated at $2~\mbox{Hz}$, using a LiDAR sensor (VLP-16 by Velodyne) mounted on the robot. We use a standard occupancy grid mapping algorithm~\citep{thrun2002probabilistic} to construct the occupancy map with robot poses provided by an OptiTrack motion capture system. An input limit of $v_{max} = [-2,2]~\mbox{m}/\mbox{s}$ and $\omega_{max} = [-2,2]~\mbox{rad}/\mbox{s}$ is enforced. The dead zone is selected to be $\rho_{min} = 0.05\,\mbox{m}$, and the maximum position error is set to $\rho_{max} = 0.5\,\mbox{m}$, considering the input limits. All processes are executed by an onboard AMD Ryzen$^{TM}$ 9 6900HX CPU computer within a ROS 1 environment. 


\subsubsection{Training Details} \label{sec:training_details}
Given that both terrains are flat, model discrepancies are primarily a function of vehicle linear and angular velocities and input time delays. To obtain training data that captures these discrepancies, we follow a ``Figure-8" trajectory, $x^d(t) = 2.5\cos(2\pi t/T_{lap})$ and $y^d(t) = 1.25\sin(4\pi t/T_{lap})$, at four different desired lap timings without obstacles, denoted as $T_{lap}$, for all four configurations. $T_{lap}$ takes values of $20$, $30$, $40$, and $50$ seconds. We gathered approximately $5$ minutes of data at each configuration, with a sampling rate of 20 Hz. A scatter plot of the identified unmatched and matched disturbances from the training data is displayed in Fig. \ref{fig:training_results} which showed non-negligible model discrepancies, mainly arising from input delays and track slipping. Subsequently, we compute the conformal-driven upper bounds and establish the maximum tracking deviation radii, presented in Table \ref{table:table1} using sub-sample count of $L = 3000$.

\subsubsection{Choice of Parameters}
The control algorithm operated at a frequency of $20$ Hz. The MPPI planning horizon is $1.5$ seconds, with an input noise covariance matrix $\sigma_{\mathbold{u}} = \mbox{diag}(0.2,0.2)$. The MPPI sample size was chosen to be $N_{sample} = 2000$. The MPPI costs in \eqref{eq:cost_standard} were selected to be $Q =  \mbox{diag}(50,50)$, $R = \mbox{diag}(1,1)$, $Q_f = \mbox{diag}(200,200)$, and $\alpha_o = \alpha_{ISS} = 10000$. The MPPI inverse temperature parameter $\lambda$ is chosen to be $0.1$. For the cost map, \mbox{$\alpha_{shift} = 0.1$}, \mbox{$w_{map} = l_{map} = 200$}, and $r_{map} = 0.05\,\mbox{m}$, with the map being stationary and centered at the position $(0,0)$. 

\begin{figure}[h!]
\centering
    \includegraphics[width=0.99\linewidth]{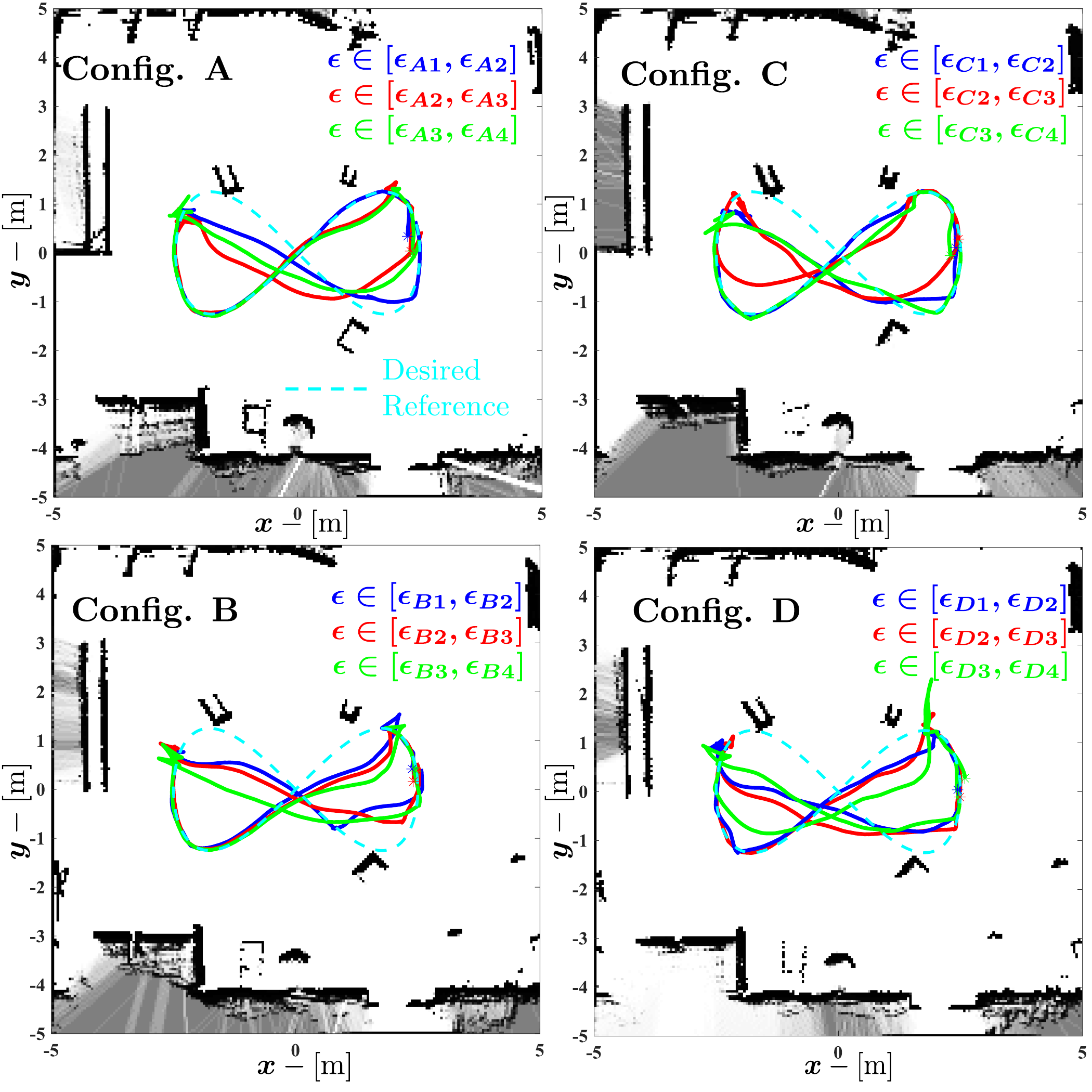}
    \caption{Experimental results of $30$-second laps ``Figure-8" tracking in the presence of stationary obstacles with test configurations A, B, C, and D. The black pixels in the backgrounds for each configuration are the obstacle-occupied grids from LiDAR measurements. The obstacles are placed to obstruct the vehicle if not avoided. The unsafe probability $\epsilon$ ranges for the four configurations are $\epsilon_{A1} = 0.1760$,  $\epsilon_{A2} = 0.0029$, $\epsilon_{A3}=0.0010$, $\epsilon_{A4} = 0.0005$, $\epsilon_{B1} = 0.1409$,  $\epsilon_{B2} = 0.0016$, $\epsilon_{B3} = 0.0013$, $\epsilon_{B4} = 0.0011$,  $\epsilon_{C1} = 0.1321$,  $\epsilon_{C2} = 0.0024$, $\epsilon_{C3} = 0.0011$, $\epsilon_{C4} = 0.0009$, $\epsilon_{D1} = 0.0031$, $\epsilon_{D2} = 0.0016$, $\epsilon_{D3} = 0.0012$, and $\epsilon_{D1} = 0.001$.}
    \label{fig:result}
\end{figure}

\subsubsection{Controller Specifics}
The effectiveness of our controller, as outlined in Theorem \ref{thm:main_thm}, is contingent on the existence of the parameter, $\tilde{\lambda}_1$. However, the time derivative of the Lyapunov function in \eqref{eq:Layp} does not depend on the polar error $\delta$, which implies that $\tilde{\lambda}_1$'s existence cannot be assured. Focusing on the primary objective of tracking the desired positions $\mathbold{p}^d(t)$, we consider the reduced polar coordinate error dynamics as follows:
\begin{equation} \label{eq:reducted_polar}
    \dot{\hat{\mathbold{e}}} = \frac{d}{dt} \begin{bmatrix}
        \rho \\
        \gamma
    \end{bmatrix}= \begin{bmatrix}
        -\cos(\gamma) &     0    \\
       \frac{\sin(\gamma)}{\rho} & -1 
    \end{bmatrix}\delta\mathbold{u} = \hat{g}_{p}(\hat{\mathbold{e}})\delta\mathbold{u}.
\end{equation}
The control input of \eqref{eq:reducted_polar} is $\mathbold{u} = \hat{\kappa}(\hat{\mathbold{e}}) =  [\hat{v}, \hat{\omega}]^{T}$, where $\hat{v} = k_1 \rho \cos{\gamma}$ and $
    \hat{\omega} = k_2 \gamma + k_1 \sin(\gamma)\cos(\gamma)$.  With the Lyapunov function $\hat{V} = \frac{1}{2}(\rho^2 + \gamma^2)$, the controlled system \eqref{eq:reducted_polar} is exponentially stabilizing to $(0,0)$ within domain $\mathcal{D}^{\rho} \setminus \mathcal{D}^{dz} \times \mathcal{D}^{\gamma}$ with $\dot{\hat{V}} \leq -\alpha_3\|\hat{\mathbold{e}}\|$ and $\alpha_3 = \frac{k_2 \rho_{dz}}{2\sqrt{(\pi/2)^2 + \rho_{dz}}}$. Choosing $k_1=0.3$ and $k_2=0.15$, the perturbed closed-loop system has a Lipschitz constant $l_{\hat{V}} = \pi/2$, which is calculated over the input domain. For tracking purposes, the convergence of $\hat{\mathbold{e}}$ to zero implies achieving the desired output tracking, i.e., $\mathbold{p} \to \mathbold{p}^*$.
    
    Applying Theorem \ref{thm:main_thm} to the reduced polar coordinate error dynamics in \eqref{eq:reducted_polar}, and with our parameter choices, we obtain $\alpha_1 = \alpha_2 = 0.5$. The class $K_{\infty}$ function $\alpha_3(V) = 0.0024V$ is affine. Most importantly, there exists $\hat{\lambda}_1 = 1000$ which leads to $\tilde{\alpha}_3(V) = -0.0038 V$. With the augmented ancillary controller $\hat{\kappa}_{ISS} = \hat{\kappa}(\hat{\mathbold{e}}) - \frac{1}{\tilde{\lambda}_1} \hat{g}_p(\hat{\mathbold{e}})^{T}\hat{\mathbold{e}}$, we compute the radii $r_0$ and $r_{\Delta T}$ for the 4 configurations consolidated, see Table \ref{table:table1}. The overall policy that is sent to the vehicle is:
    \begin{equation}\label{eq:ucmd}
        \mathbold{u}_{cmd} = \min(\max(\mathbold{u}_0^* + \hat{\kappa}_{ISS},\mathbold{u_{min}}),\mathbold{u}_{max}).
    \end{equation}

\begin{table}
\begin{center} 
\resizebox{\columnwidth}{!}{
\begin{tabular}{ |c|c|c|c|c| } \hline 
\backslashbox{Config}{$\epsilon$} & Param. & $0.001$ &$0.005$ &$0.01$\\ \hline 
\multirow{2}*{A} & $Z_{\epsilon}\,|\,Z_{\epsilon}^{\perp}$  &$0.710\,|\,0.448$  & $0.423\,|\,0.025$& $0.393\,|\,0.019$ \\
& $r_0\,|\, r_{\Delta T}$& $0.252\,|\,0.255$ & $0.090\,|\,0.090$ & $0.077\,|\,0.077$\\ \hline 

\multirow{2}*{B} & $Z_{\epsilon}\,|\,Z_{\epsilon}^{\perp}$ &$2.153\,|\,0.034$ &$0.419\,|\,0.030$ &$0.381\,|\, 0.026$\\
& $r_0\,|\, r_{\Delta T}$ & $2.318\,|\,2.320$ & $0.088\,|\,0.088$& $0.073\,|\,0.073$\\ \hline 

\multirow{2}*{C} & $Z_{\epsilon}\,|\,Z_{\epsilon}^{\perp}$ & $1.113\,|\,0.032$ & $0.413\,|\,0.027$& $0.369\,|\,0.020$ \\
 & $r_0| r_{\Delta T}$& $0.619\,|\,0.620$& $0.085\,|\,0.085$ & $0.068\,|\,0.068$\\ \hline 

\multirow{2}*{D} & $Z_{\epsilon}\,|\,Z_{\epsilon}^{\perp}$ & $1.878\,|\,0.095$ & $0.429\,|\,0.032$ & $0.401\,|\,0.031$  \\
 & $r_0\,|\, r_{\Delta T}$& $1.763\,|\,1.768$ & $0.092\,|\,0.092$ & $0.080\,|\,0.081$\\ \hline
\end{tabular}
}
\caption{Summary of the offline conformal discrepancy training results and the augmented controller tracking guarantees where autonomous trajectory tracking is performed without obstacles in the four configurations. The training $1-\epsilon$ confidence upper bounds for $\|\mathbold{d}_{\mathbold{u}}\|$ and $|d{d}_{\perp}|$ is provided at three $\epsilon$ levels using Algorithm \ref{alg:1}. Under the augmented control policy ${\kappa}_{ISS}$, we tabulated $r_0$ and $r_{\Delta T}$ with the choices of $\epsilon$ levels. Note, the identified radii $r_0$, $r_{\Delta t}$ are in meters.}
\label{table:table1}
\end{center}
\end{table}

\subsubsection{Results}
The discrepancy-aware MPPI planner is validated by tracking the same ``Figure 8" trajectory used for collecting training data, denoted as $\mathbold{p}^d = (x^d(t), y^d(t))$, with a lap time of $T_{lap} = 30~\mbox{sec}$\footnote{Despite the validation desired trajectory's geometry matches with the training set, the vehicle must deviate from the reference trajectory to avoid the stationary obstacles, covering untrained positions}. The experimental setup can be found in Fig. \ref{fig:maps}. As highlighted in the supplementary video \citep{supp_video}, the nominal MPPI planner combined with the nominal controller without accounting model mismatches, failed to ensure safety, resulting in a collision with the stationary crates. In contrast, our proposed framework completed the trajectory tracking tasks across all four vehicle-ground configurations, effectively handling model mismatches while providing a verifiable safe traversal confidence level. The experimental results are summarized in Fig. \ref{fig:result} for the four vehicle-ground configurations of Fig. \ref{fig:schematics}.

A key feature of our approach is the augmentation of the cost map on a grid basis.  We associate each increment in $N_{\epsilon}$ with a specific range of $\epsilon$ values, demonstrating the tunability of our framework. It was observed that lower values of $\epsilon$ (approaching zero) induce a conservative obstacle-avoiding behavior, akin to traditional robust control methods.  Conversely, as $\epsilon$ nears one, the risk-neutral case, the vehicle tracks closely with the reference path, shows improvements in tracking performance at the cost of safety confidence.

During our analysis, particularly with configurations B and D of Fig. \ref{fig:schematics}, large model discrepancies for $\epsilon = 0.001$ are observed. These discrepancies were predominantly found in datasets characterized by lower lap speeds, suggesting a correlation with specific state or input conditions. This observation signals the possibilities of a state and input-dependent discrepancy discussed in \citep{pmlr_v211_akella23a}, which could potentially provide a more accurate safety buffer compared to our upper bounds $Z_{\epsilon}$ and $Z_{\epsilon}^{\perp}$.

\subsection{Driver Assist for Collision Avoidance} 

\begin{figure*}[h!]
\centering
    \includegraphics[width=0.99\linewidth]{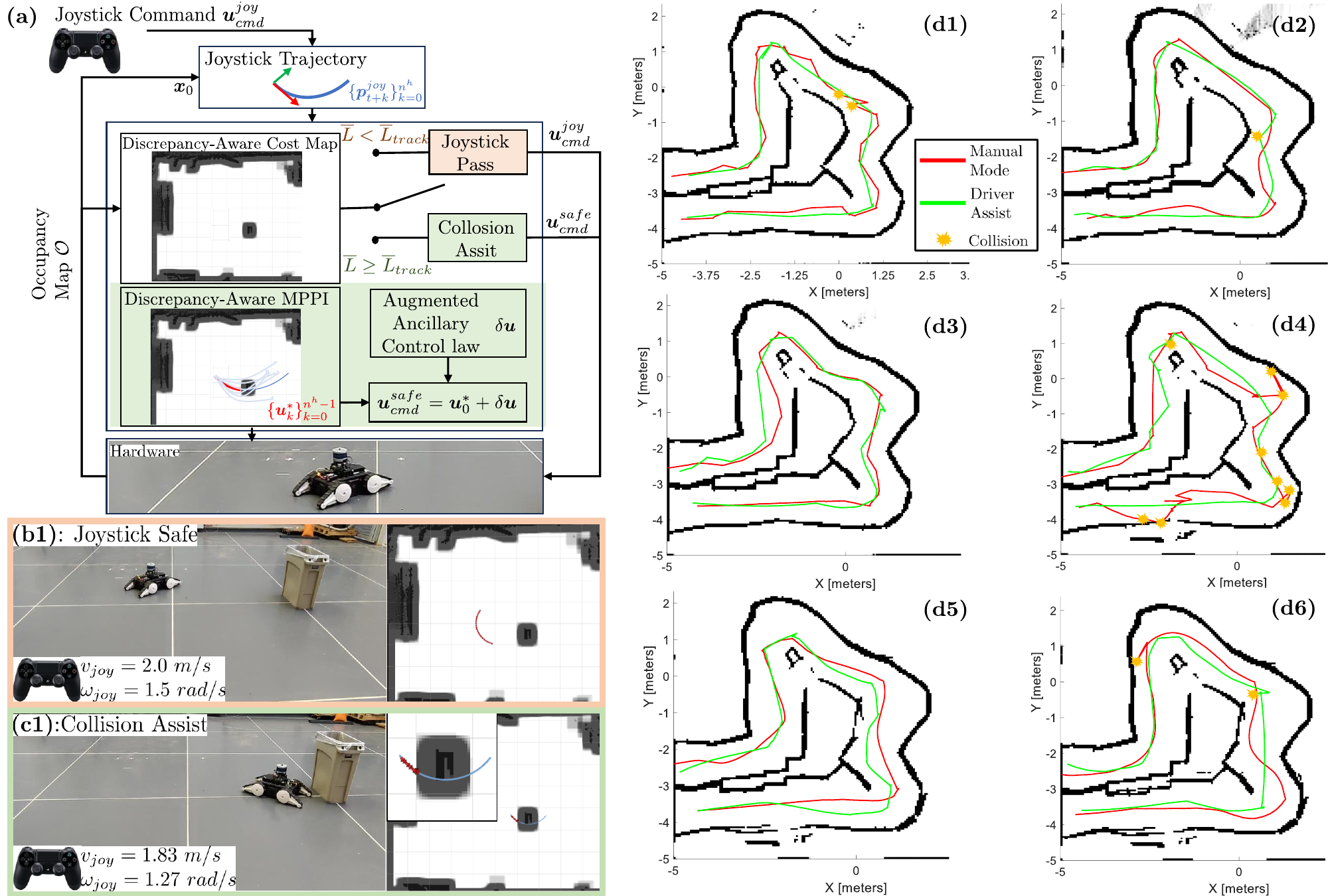}
    \caption{Experimental results of using the proposed framework performing collision avoidance assistance for human drivers. Based on human-provided joystick commands and a user-chosen risk tolerance $\epsilon$, the collision risk of joystick trajectory based on current state feedback is evaluated on the proposed discrepancy-aware cost map. If the collision is projected based on the evaluated cost, the proposed planning and control framework will activate to provide overriding safe commands. Sub-figure (a) is a diagram that summarizes the driver-assist logic flows. Sub-figures (b1) and (b2) illustrate the scenario for safe joystick inputs where no collision assist is inactive. Sub-figures (c1) and (c2) illustrate the scenario where the projected joystick trajectory is unsafe and where the collision assist program is active to provide optimal and safe commands. Sub-figures (d1)-(d6) are the six test subjects' trajectories in manual mode and driver assist mode.}
    \label{fig:result_driver_assist}
\end{figure*}
\subsubsection{Human-in-the-Loop Setup}
This section explores the application of our proposed framework in driver-assist scenarios. A driver remotely operates a UGV robot via a joystick. The driver controls the vehicle's linear velocity in the body frame's $x$ axis ($v_{joy}$) and angular velocity in the $z$ axis ($\omega_{joy}$), the same as the input commands in \eqref{eq:origin}. The driver assist program aims to follow the joystick commands as closely as possible while avoiding obstacles up to the specified risk tolerance.

Drawing from Hugemann's study on driver reaction times in road traffic \citep{Hugemann2002DriverRT}, we anticipate a, rather conservative, human reaction time of $1.5$ seconds for collision avoidance. We interpret this $1.5$ seconds reaction time as a $1.5$ second ZOH to the most recent joystick commands. Specifically, we replace the desired reference trajectory with the joystick trajectory $\{\mathbold{x}_{t+k}\}_{k=1}^{n_h}$, formulated as 
\begin{equation}
    \mathbold{x}_{t+k}^{joy} = \mathbold{x}_{t} + \sum_{i=0}^{k}  g(\mathbold{x}_{t+i})\begin{bmatrix}
        v_{joy}\\
        \omega_{joy}
    \end{bmatrix}\Delta T,
\end{equation}
where $k\in \mathbb{Z}_{1}^{n_h}$. To ensure maximum teleoperation control in collision-free environments, we only provide trajectory correction when a collision is imminent based on the joystick trajectory.

In our experiment, human subjects drive the UGV in configuration A without a direct view of the test course. The test drivers rely solely on a 2D occupancy map for navigating among obstacles, matching the perception capability of the framework. The human drivers are tasked to finish the narrow test course without collision as quickly as possible.   

Note, the same training set presented in section \ref{sec:training_details} is used in the driver assist application despite the vehicle trajectory does not match the training Figure 8 trajectory to highlight the learned upper bounds $Z_{\epsilon}$ and $Z_{\epsilon}^{\perp}$ is training trajectory independent. 

\subsubsection{Driver-Assist Implementation}
Using the training result tabulated in table \ref{table:table1}, we construct the discrepancy-aware cost map with a risk value $\epsilon$ ranges in $(0.0029,0.0010)$. We determine a joystick trajectory cost:
\begin{equation}
    \overline{\mathcal{L}}_{joy} = \sum_{k=1}^{n_h} \alpha_{joy}\mathcal{L}_{\mathcal{C}}(\mathbold{p}_{t+k}^{joy} , \mathcal{C}_{\epsilon}),
\end{equation}
where $\alpha_{joy}\in \mathbb{R}_{>0}$ is an adjustable parameter with default value of $1.0$. The cost-to-go and input cost are set to zero, considering the joystick trajectory itself is the desired reference. If $\overline{\mathcal{L}}_{joy}$ is less than $\overline{\mathcal{L}}_{track}$, the joystick trajectory is deemed safe with a confidence level of $1-\epsilon$ for the immediate $1.5$ seconds. Conversely, if $\overline{\mathcal{L}}_{joy} \geq \overline{\mathcal{L}}_{track}$, there is a potential collision, requiring a safety override.

The proposed framework implements the safety override. We initialized the MPPI algorithm using the joystick input as a warm start, with input perturbation covariance set as a diagonal matrix of $0.25~m/s$ and $~0.25 rad/s$ for linear and angular velocities. Since the maximum joystick linear and angular speeds are $2~m/s$ and $2~rad/s$, a zero-speed (stopping) command cannot be guaranteed to be sampled. As a solution, we also add turn-in-place (TIP) motion primitives into the flow sampling set and MPPI input aggregation, in addition to the joystick trajectory. Specifically, the joystick turning command with $0$ linear velocity, $\{\mathbold{u}_{k}^{TIP}\}_{k=0}^{k=n_h-1} = [0,\omega_{joy}]^{T}$, allow the driver to perform heading adjustments freely, but not the vehicle positions because of potential collision risks in the current heading direction. 

With a total MPPI sample size of $N_{sample} = 5000$ and inverse temperature $\lambda = 0.05$, we allocated and aggregated $80\%$ of sampled input discrepancies to the joystick input sequences and the remaining samples to the TIP motion primitives. Following Algorithm \eqref{alg:2} to compute the associated costs of each sampled flow, we obtain the collision-assist safety override $\mathbold{u}_{cmd}^{safe}$. \footnote{We always use the joystick trajectory as a warm start without using the previous MPPI solution. Such engineering decision is made to ensure maximum adherence to the driver's command which can be discontinuous and far from the previous projected joystick trajectory.} A detailed diagram of the collision-assist program is presented in Fig. \eqref{fig:result_driver_assist}.

We select parameter $\alpha_{joy}$ to be $1/{k^{2}}$, inversely proportional to the prediction horizon squared. As the collision-assistance program recalculates at $20$ Hz, this cost decay over future horizon steps is chosen to reduce planner conservatism. 

\subsubsection{Results and Discussion} We conducted an experimental comparison with six human drivers navigating a narrow pathway, with various experience levels with joystick operation. Each test subject drives the same test course twice, once with and once without the proposed collision assistive program. The order of the testing is randomized to minimize any bias due to the driver's familiarity with the test course. The test results are tabulated in Table \ref{table:table_driver_assit}, with comparative drive trajectories depicted in Fig. \ref{fig:result_driver_assist}.

\begin{table}
\begin{center} 
\resizebox{\columnwidth}{!}{
\begin{tabular}{ |c|c|c|c|c|c|} \hline 
Subj. & Seq. & $\#$ C (M) & $T_{t}$ (M)  & $\#$ C (CA) &  $T_{t}$ (CA) \\ 
\hline 
{1} & M/DA & 2 & 63.7 & 0 &38.8  \\ \hline

{2} & M/DA &0 & 50.7 &1 & 40.3 \\ \hline

{3} & DA/M &0 & 41.5 &0 & 44.8 \\ \hline

{4} & M/DA &7 & 107.5 &3 & 63.6   \\ \hline

{5} & DA/M & 0& 45.6 &0 & 30.1 \\ \hline

{6} & DA/M & 1& 35.6 &1 & 28.3  \\ \hline

\end{tabular}
}

\caption{Summary of the diver assist program with 6 different human drivers. Each subject is assigned to attempt to complete the task with either manual mode (M) or the collision assistance mode (CA). The mode experience order is indicated in the second column. The number of collisions ($\# C$) during the test drives is recorded and the test drive duration $T_t$.}
\label{table:table_driver_assit}
\end{center}
\end{table}

Based on the tabulated results, we can see the proposed collision-assistive program showed minor improvements for rather inexperienced drivers (such as drivers 1 and 4) who faced challenges during their initial manual run. Despite the driving sequence being assigned randomly, a consistent reduction in driving time was observed on their second pass, likely attributable to the accumulation of track and vehicle knowledge. Notably, subjects who initially experienced manual driving followed by the driver-assist program exhibited percent drive time improvements of $39.1\%$, $20.5\%$, and $40.8\%$. Conversely, subjects who first used the assistive program and then switched to manual mode showed improvements of $-8.0\%$, $34.0\%$, and $20.5\%$. This variance in performance improvement may be attributed to enhanced driver confidence in collision avoidance due to the assistive program.

The oral feedback from test subjects on the driver-assist program was mixed. Positive remarks primarily center around the program’s effectiveness and the ability to provide collision-avoiding heading adjustments automatically. However, several drivers noted counterintuitive instances where the vehicle, under driver-assist control, executed minor reverse maneuvers despite receiving forward velocity commands. This behavior is the result of the model mismatches, where the nominal model predicts a trajectory leading into high-risk or collision-prone areas, prompting the program to revert to a safer region. Furthermore, integrating user-preference-based cost tuning could better tailor individual drivers' preferences in balancing between reference tracking and collision safety, as suggested in recent studies \citep{pmlr-v168-cosner22a, tucker2022polar}. It is important to note that due to the limited number of participants, we refrain from drawing statistically significant conclusions from this study.

\subsection{Limitation and Future Works}
The current limitations of the proposed framework include potential safety violations in the importance sampling aggregation law \citep{yin2023shield}, a tendency towards planning overly conservative trajectories, and vulnerability to training outliers in model discrepancies. Based on the choice of the inverse temperature $\lambda$, $\lambda\to \infty$ plans smoother trajectories with equal weighting, and $\lambda \to 0$ equates to sample rejection that could lead to instability \citep{8558663}. However, sampled input aggregation as $\lambda\to \infty$ could lead to an unsafe aggregated trajectory where safe and unsafe trajectories are less distinguished. Nevertheless, $\lambda \to 0$ prioritizes picking safe trajectories but may lead to overly conservative and chattered trajectories from randomly sampled input sequences. 

On the discrepancy identification front, conformal prediction relies on empirical distribution quantiles instead of concentration inequalities, allowing it to be more sample-efficient. However, in the presence of a large number of samples, potential outliers, and small risk tolerance, the identified $Z_{\epsilon}$ and $Z_{\epsilon}^{\perp}$ can be falsely large, yielding overly conservative results. Sufficient data set prepossessing might be required to remove outliers.

Looking ahead, our future work aims to refine the offline conformal-driven discrepancy upper bound analysis into an online algorithm. We are also exploring adaptive conformal prediction methods, akin to the one detailed in \citep{dixit2023adaptive}, to actively detect distribution shifts from the training sets due to terrain or environmental changes. Moreover, a higher fidelity nominal model such as a learning-based model can be used to enhance tracking performance. Another future research involves the theoretical analysis of the discrepancy-aware MPPI in terms of optimality and recursive feasibility. The authors are also investigating the possibilities of rejecting unsafe MPPI trajectory samples, incorporating motion primitives, and incorporating generalized state constraints beyond obstacle avoidance constraints.

\section{Conclusion}\label{sec:conclusion}
In summary, this manuscript provides a novel multi-layered framework that is designed to provide safety-critical autonomy in the presence of obstacles and model discrepancies. The discrepancies include poorly modeled terrain interactions, system delays, and simplified dynamic models. The framework's core strategy involves data-driven discrepancy identification, extracting both matched and unmatched model residuals from offline data, with minimal assumptions. These identified discrepancies are then used to augment the vehicle's ancillary controller, offering stability assurances for the closed-loop system. We then complete the framework with a discrepancy-aware MPPI planner that generates (sub)optimal and safe reference tracking paths, taking into account imperfection in actual trajectory tracking due to model discrepancies. 
 
 Our proposed framework is theoretically supported throughout its construction. By deducing the maximum tracking error resulting from matched and unmatched model discrepancies, we ensure safety and robustness by assessing the interactions between the planner and controller layers. We also validated the proposed framework through extensive hardware experiments, demonstrating its effectiveness in trajectory tracking in cluttered environments. Additionally, we have successfully adapted the framework as a driver-assist program, providing optimal, safe assistive commands in potential collision scenarios.

\section*{Acknowledgements}
This work was supported by DARPA under the Learning and Introspective Control (LINC) program and by the Technology Innovation Institute (TII) through a grant to the Center for Autonomous Systems and Technologies.

The authors thank Anushri Dixit and Thomas Tomar for their help in configuring the Flipper Pro. We thank Kejun Li, William Welch, Victor Zendejas Lopez, Mandralis Ioannis, and Lizhi Yang for helping with the experiments.



\bibliographystyle{elsarticle-num} 
 \bibliography{references}

\begin{thebibliography}{10}
\expandafter\ifx\csname url\endcsname\relax
  \def\url#1{\texttt{#1}}\fi
\expandafter\ifx\csname urlprefix\endcsname\relax\def\urlprefix{URL }\fi
\expandafter\ifx\csname href\endcsname\relax
  \def\href#1#2{#2} \def\path#1{#1}\fi

\bibitem{bonadies2016survey}
S.~Bonadies, A.~Lefcourt, S.~A. Gadsden, A survey of unmanned ground vehicles with applications to agricultural and environmental sensing, in: Autonomous air and ground sensing systems for agricultural optimization and phenotyping, Vol. 9866, SPIE, 2016, pp. 142--155.

\bibitem{9197082}
K.~Ebadi, Y.~Chang, M.~Palieri, A.~Stephens, A.~Hatteland, E.~Heiden, A.~Thakur, N.~Funabiki, B.~Morrell, S.~Wood, L.~Carlone, A.-a. Agha-mohammadi, Lamp: Large-scale autonomous mapping and positioning for exploration of perceptually-degraded subterranean environments, in: 2020 IEEE International Conference on Robotics and Automation (ICRA), 2020, pp. 80--86.
\newblock \href {https://doi.org/10.1109/ICRA40945.2020.9197082} {\path{doi:10.1109/ICRA40945.2020.9197082}}.

\bibitem{wu2021autonomous}
Y.~Wu, Y.~Ding, S.~Ding, Y.~Savaria, M.~Li, Autonomous last-mile delivery based on the cooperation of multiple heterogeneous unmanned ground vehicles, Mathematical Problems in Engineering 2021 (2021) 1--15.

\bibitem{9172811}
J.~Thangavelautham, A.~Chandra, E.~Jensen, Autonomous robot teams for lunar mining base construction and operation, in: 2020 IEEE Aerospace Conference, 2020, pp. 1--16.
\newblock \href {https://doi.org/10.1109/AERO47225.2020.9172811} {\path{doi:10.1109/AERO47225.2020.9172811}}.

\bibitem{akella2022sample}
P.~Akella, A.~Dixit, M.~Ahmadi, J.~W. Burdick, A.~D. Ames, Sample-based bounds for coherent risk measures: Applications to policy synthesis and verification, arXiv:2204.09833 (2022).

\bibitem{saveriano2017data}
M.~Saveriano, Y.~Yin, P.~Falco, D.~Lee, Data-efficient control policy search using residual dynamics learning, in: 2017 Int. Conference on Intelligent Robots and Systems (IROS), 2017, pp. 4709--4715.

\bibitem{dixit2023step}
A.~Dixit, D.~D. Fan, K.~Otsu, S.~Dey, A.-A. Agha-Mohammadi, J.~W. Burdick, Step: Stochastic traversability evaluation and planning for risk-aware off-road navigation; results from the darpa subterranean challenge, arXiv preprint arXiv:2303.01614 (2023).

\bibitem{bouman2020autonomous}
A.~Bouman, M.~F. Ginting, N.~Alatur, M.~Palieri, D.~D. Fan, T.~Touma, T.~Pailevanian, S.-K. Kim, K.~Otsu, J.~Burdick, A.~akbar Agha-mohammadi, Autonomous spot: Long-range autonomous exploration of extreme environments with legged locomotion (2020).
\newblock \href {http://arxiv.org/abs/2010.09259} {\path{arXiv:2010.09259}}.

\bibitem{9813568}
L.~Gan, J.~W. Grizzle, R.~M. Eustice, M.~Ghaffari, Energy-based legged robots terrain traversability modeling via deep inverse reinforcement learning, IEEE Robotics and Automation Letters 7~(4) (2022) 8807--8814.
\newblock \href {https://doi.org/10.1109/LRA.2022.3188100} {\path{doi:10.1109/LRA.2022.3188100}}.

\bibitem{williams2018robust}
G.~Williams, B.~Goldfain, P.~Drews, K.~Saigol, J.~M. Rehg, E.~A. Theodorou, Robust sampling based model predictive control with sparse objective information., in: Robotics: Science and Systems, Vol.~14, 2018, p. 2018.

\bibitem{mppi_original}
G.~Williams, N.~Wagener, B.~Goldfain, P.~Drews, J.~M. Rehg, B.~Boots, E.~A. Theodorou, Information theoretic mpc for model-based reinforcement learning, in: 2017 IEEE International Conference on Robotics and Automation (ICRA), 2017, pp. 1714--1721.
\newblock \href {https://doi.org/10.1109/ICRA.2017.7989202} {\path{doi:10.1109/ICRA.2017.7989202}}.

\bibitem{sinha2022adaptive}
R.~Sinha, J.~Harrison, S.~M. Richards, M.~Pavone, Adaptive robust model predictive control with matched and unmatched uncertainty, in: 2022 American Control Conference (ACC), IEEE, 2022, pp. 906--913.

\bibitem{1323177}
W.-J. Cao, J.-X. Xu, Nonlinear integral-type sliding surface for both matched and unmatched uncertain systems, IEEE Transactions on Automatic Control 49~(8) (2004) 1355--1360.
\newblock \href {https://doi.org/10.1109/TAC.2004.832658} {\path{doi:10.1109/TAC.2004.832658}}.

\bibitem{pravitra2020}
J.~Pravitra, K.~A. Ackerman, C.~Cao, N.~Hovakimyan, E.~A. Theodorou, $l_1$-adaptive mppi architecture for robust and agile control of multirotors, in: 2020 IEEE/RSJ International Conference on Intelligent Robots and Systems (IROS), IEEE, 2020, pp. 7661--7666.

\bibitem{CP_planning_Lindemann}
L.~Lindemann, M.~Cleaveland, G.~Shim, G.~J. Pappas, Safe planning in dynamic environments using conformal prediction, IEEE Robotics and Automation Letters 8~(8) (2023) 5116--5123.
\newblock \href {https://doi.org/10.1109/LRA.2023.3292071} {\path{doi:10.1109/LRA.2023.3292071}}.

\bibitem{alan2021safe}
A.~Alan, A.~J. Taylor, C.~R. He, G.~Orosz, A.~D. Ames, Safe controller synthesis with tunable input-to-state safe control barrier functions, IEEE Control Systems Letters 6 (2021) 908--913.

\bibitem{NoelCS_multirate}
N.~Csomay-Shanklin, A.~Taylor, U.~Rosolia, A.~Ames, Multi-rate planning and control of uncertain nonlinear systems: Model predictive control and control lyapunov functions (03 2022).

\bibitem{majumdar2020should}
A.~Majumdar, M.~Pavone, How should a robot assess risk? towards an axiomatic theory of risk in robotics, in: Robotics Research, Springer, 2020, pp. 75--84.

\bibitem{dixit2021risk}
A.~Dixit, M.~Ahmadi, J.~W. Burdick, Risk-sensitive motion planning using entropic value-at-risk, in: 2021 European Control Conference (ECC), IEEE, 2021, pp. 1726--1732.

\bibitem{cakmak2020bayesian}
S.~Cakmak, R.~Astudillo~Marban, P.~Frazier, E.~Zhou, Bayesian optimization of risk measures, Advances in Neural Information Processing Systems 33 (2020) 20130--20141.

\bibitem{makarova2021risk}
A.~Makarova, I.~Usmanova, I.~Bogunovic, A.~Krause, Risk-averse heteroscedastic bayesian optimization, Advances in Neural Information Processing Systems 34 (2021) 17235--17245.

\bibitem{o2022neural}
M.~O’Connell, G.~Shi, X.~Shi, K.~Azizzadenesheli, A.~Anandkumar, Y.~Yue, S.-J. Chung, Neural-fly enables rapid learning for agile flight in strong winds, Science Robotics 7~(66) (2022).

\bibitem{bemporad2007robust}
A.~Bemporad, M.~Morari, Robust model predictive control: A survey, in: Robustness in identification and control, Springer, 2007, pp. 207--226.

\bibitem{lopez2019dynamic}
B.~T. Lopez, J.-J.~E. Slotine, J.~P. How, Dynamic tube mpc for nonlinear systems, in: 2019 American Control Conference (ACC), IEEE, 2019, pp. 1655--1662.

\bibitem{nakka2022trajectory}
Y.~K. Nakka, S.-J. Chung, Trajectory optimization of chance-constrained nonlinear stochastic systems for motion planning under uncertainty, IEEE Transactions on Robotics 39~(1) (2022) 203--222.

\bibitem{8613928}
H.~Zhu, J.~Alonso-Mora, Chance-constrained collision avoidance for mavs in dynamic environments, IEEE Robotics and Automation Letters 4~(2) (2019) 776--783.
\newblock \href {https://doi.org/10.1109/LRA.2019.2893494} {\path{doi:10.1109/LRA.2019.2893494}}.

\bibitem{5477242}
L.~Blackmore, M.~Ono, A.~Bektassov, B.~C. Williams, A probabilistic particle-control approximation of chance-constrained stochastic predictive control, IEEE Transactions on Robotics 26~(3) (2010) 502--517.
\newblock \href {https://doi.org/10.1109/TRO.2010.2044948} {\path{doi:10.1109/TRO.2010.2044948}}.

\bibitem{5970128}
L.~Blackmore, M.~Ono, B.~C. Williams, Chance-constrained optimal path planning with obstacles, IEEE Transactions on Robotics 27~(6) (2011) 1080--1094.
\newblock \href {https://doi.org/10.1109/TRO.2011.2161160} {\path{doi:10.1109/TRO.2011.2161160}}.

\bibitem{calafiore2013stochastic}
G.~C. Calafiore, L.~Fagiano, Stochastic model predictive control of lpv systems via scenario optimization, Automatica 49~(6) (2013) 1861--1866.

\bibitem{9349120}
M.~S. Gandhi, B.~Vlahov, J.~Gibson, G.~Williams, E.~A. Theodorou, Robust model predictive path integral control: Analysis and performance guarantees, IEEE Robotics and Automation Letters 6~(2) (2021) 1423--1430.
\newblock \href {https://doi.org/10.1109/LRA.2021.3057563} {\path{doi:10.1109/LRA.2021.3057563}}.

\bibitem{occu_map_1}
S.~Thrun, \href{https://doi.org/10.1145/504729.504754}{Probabilistic robotics}, Commun. ACM 45~(3) (2002) 52–57.
\newblock \href {https://doi.org/10.1145/504729.504754} {\path{doi:10.1145/504729.504754}}.
\newline\urlprefix\url{https://doi.org/10.1145/504729.504754}

\bibitem{occu_map_2}
S.~Thrun, A.~B{\"u}cken, Integrating grid-based and topological maps for mobile robot navigation, in: Proceedings of the national conference on artificial intelligence, 1996, pp. 944--951.

\bibitem{tsardoulias2016review}
E.~G. Tsardoulias, A.~Iliakopoulou, A.~Kargakos, L.~Petrou, A review of global path planning methods for occupancy grid maps regardless of obstacle density, Journal of Intelligent \& Robotic Systems 84 (2016) 829--858.

\bibitem{Oleynikova2016SignedDF}
H.~Oleynikova, A.~Millane, Z.~Taylor, E.~Galceran, J.~I. Nieto, R.~Y. Siegwart, \href{https://api.semanticscholar.org/CorpusID:28083959}{Signed distance fields: A natural representation for both mapping and planning}, 2016.
\newline\urlprefix\url{https://api.semanticscholar.org/CorpusID:28083959}

\bibitem{camps2022learning}
G.~S. Camps, R.~Dyro, M.~Pavone, M.~Schwager, Learning deep sdf maps online for robot navigation and exploration (2022).
\newblock \href {http://arxiv.org/abs/2207.10782} {\path{arXiv:2207.10782}}.

\bibitem{Fan_David_costmap}
D.~D. Fan, A.-a. Agha-mohammadi, E.~A. Theodorou, Learning risk-aware costmaps for traversability in challenging environments, IEEE Robotics and Automation Letters 7~(1) (2022) 279--286.
\newblock \href {https://doi.org/10.1109/LRA.2021.3125047} {\path{doi:10.1109/LRA.2021.3125047}}.

\bibitem{cai2023evora}
X.~Cai, S.~Ancha, L.~Sharma, P.~R. Osteen, B.~Bucher, S.~Phillips, J.~Wang, M.~Everett, N.~Roy, J.~P. How, Evora: Deep evidential traversability learning for risk-aware off-road autonomy (2023).
\newblock \href {http://arxiv.org/abs/2311.06234} {\path{arXiv:2311.06234}}.

\bibitem{Lakshay}
L.~Sharma, M.~Everett, D.~Lee, X.~Cai, P.~Osteen, J.~How, Ramp: A risk-aware mapping and planning pipeline for fast off-road ground robot navigation (09 2022).
\newblock \href {https://doi.org/10.48550/arXiv.2210.06605} {\path{doi:10.48550/arXiv.2210.06605}}.

\bibitem{yin2023shield}
J.~Yin, C.~Dawson, C.~Fan, P.~Tsiotras, Shield model predictive path integral: A computationally efficient robust mpc approach using control barrier functions (2023).
\newblock \href {http://arxiv.org/abs/2302.11719} {\path{arXiv:2302.11719}}.

\bibitem{reward_hacking}
J.~Clark, D.~Amodei, \href{https://openai.com/research/faulty-reward-functions}{Faulty reward functions in the wild}.
\newline\urlprefix\url{https://openai.com/research/faulty-reward-functions}

\bibitem{frey2023fast}
J.~Frey, M.~Mattamala, N.~Chebrolu, C.~Cadena, M.~Fallon, M.~Hutter, Fast traversability estimation for wild visual navigation, arXiv preprint arXiv:2305.08510 (2023).

\bibitem{klancar2017wheeled}
G.~Klancar, A.~Zdesar, S.~Blazic, I.~Skrjanc, Wheeled mobile robotics: from fundamentals towards autonomous systems, Butterworth-Heinemann, 2017.

\bibitem{388294}
M.~Aicardi, G.~Casalino, A.~Bicchi, A.~Balestrino, Closed loop steering of unicycle like vehicles via lyapunov techniques, IEEE Robotics and Automation Magazine 2~(1) (1995) 27--35.

\bibitem{mppi_quadcopter}
I.~S. Mohamed, G.~Allibert, P.~Martinet, Model predictive path integral control framework for partially observable navigation: A quadrotor case study, in: International Conference on Control, Automation, Robotics and Vision, 2020, pp. 196--203.
\newblock \href {https://doi.org/10.1109/ICARCV50220.2020.9305363} {\path{doi:10.1109/ICARCV50220.2020.9305363}}.

\bibitem{lindqvist2020nonlinear}
B.~Lindqvist, S.~S. Mansouri, A.-a. Agha-mohammadi, G.~Nikolakopoulos, Nonlinear mpc for collision avoidance and control of uavs with dynamic obstacles, IEEE robotics and automation letters 5~(4) (2020) 6001--6008.

\bibitem{diehl2004robust}
M.~Diehl, J.~Bjornberg, Robust dynamic programming for min-max model predictive control of constrained uncertain systems, IEEE Transactions on Automatic Control 49~(12) (2004) 2253--2257.

\bibitem{DeLuca2001}
A.~De~Luca, G.~Oriolo, M.~Vendittelli, Control of Wheeled Mobile Robots: An Experimental Overview, Springer Berlin Heidelberg, Berlin, Heidelberg, 2001, pp. 181--226.

\bibitem{Khalil_2002}
H.~K. Khalil, Nonlinear systems, Prentice Hall, 2002.

\bibitem{CHAKRABORTY20041273}
N.~Chakraborty, A.~Ghosal, \href{https://www.sciencedirect.com/science/article/pii/S0094114X04001478}{Kinematics of wheeled mobile robots on uneven terrain}, Mechanism and Machine Theory 39~(12) (2004) 1273--1287, 11th National Conference on Machines and Mechanisms (NaCoMM-2003).
\newblock \href {https://doi.org/https://doi.org/10.1016/j.mechmachtheory.2004.05.016} {\path{doi:https://doi.org/10.1016/j.mechmachtheory.2004.05.016}}.
\newline\urlprefix\url{https://www.sciencedirect.com/science/article/pii/S0094114X04001478}

\bibitem{cp_book}
V.~Vovk, A.~Gammerman, G.~Shafer, Algorithmic Learning in a Random World, Springer-Verlag, Berlin, Heidelberg, 2005.

\bibitem{dixit2023adaptive}
A.~Dixit, L.~Lindemann, S.~X. Wei, M.~Cleaveland, G.~J. Pappas, J.~W. Burdick, Adaptive conformal prediction for motion planning among dynamic agents, in: Learning for Dynamics and Control Conference, PMLR, 2023, pp. 300--314.

\bibitem{8558663}
G.~Williams, P.~Drews, B.~Goldfain, J.~M. Rehg, E.~A. Theodorou, Information-theoretic model predictive control: Theory and applications to autonomous driving, IEEE Transactions on Robotics 34~(6) (2018) 1603--1622.
\newblock \href {https://doi.org/10.1109/TRO.2018.2865891} {\path{doi:10.1109/TRO.2018.2865891}}.

\bibitem{angelopoulos2021gentle}
A.~N. Angelopoulos, S.~Bates, A gentle introduction to conformal prediction and distribution-free uncertainty quantification, arXiv preprint arXiv:2107.07511 (2021).

\bibitem{tibshirani2019conformal}
R.~J. Tibshirani, R.~Foygel~Barber, E.~Candes, A.~Ramdas, Conformal prediction under covariate shift, Advances in neural information processing systems 32 (2019).

\bibitem{sadinle2019least}
M.~Sadinle, J.~Lei, L.~Wasserman, Least ambiguous set-valued classifiers with bounded error levels, Journal of the American Statistical Association 114~(525) (2019) 223--234.

\bibitem{7487277}
G.~Williams, P.~Drews, B.~Goldfain, J.~M. Rehg, E.~A. Theodorou, Aggressive driving with model predictive path integral control, in: 2016 IEEE International Conference on Robotics and Automation (ICRA), 2016, pp. 1433--1440.
\newblock \href {https://doi.org/10.1109/ICRA.2016.7487277} {\path{doi:10.1109/ICRA.2016.7487277}}.

\bibitem{ren2023robots}
A.~Z. Ren, A.~Dixit, A.~Bodrova, S.~Singh, S.~Tu, N.~Brown, P.~Xu, L.~Takayama, F.~Xia, J.~Varley, Z.~Xu, D.~Sadigh, A.~Zeng, A.~Majumdar, Robots that ask for help: Uncertainty alignment for large language model planners (2023).
\newblock \href {http://arxiv.org/abs/2307.01928} {\path{arXiv:2307.01928}}.

\bibitem{GB_inequality}
T.~H. Gronwall, Note on the derivatives with respect to a parameter of the solutions of a system of differential equations, Annals of Mathematics 20~(4) (1919) 292--296.

\bibitem{bastos2021energy}
G.~Bastos~Jr, E.~Franco, Energy shaping dynamic tube-mpc for underactuated mechanical systems, Nonlinear Dynamics 106~(1) (2021) 359--380.

\bibitem{thrun2002probabilistic}
S.~Thrun, Probabilistic robotics, Communications of the ACM 45~(3) (2002) 52--57.

\bibitem{supp_video}
\href{https://youtu.be/0tZagFMfodI}{Supplementary video}.
\newline\urlprefix\url{https://youtu.be/0tZagFMfodI}

\bibitem{pmlr_v211_akella23a}
P.~Akella, S.~X. Wei, J.~W. Burdick, A.~Ames, Learning disturbances online for risk-aware control: Risk-aware flight with less than one minute of data, in: Proceedings of The 5th Annual Learning for Dynamics and Control Conference, Vol. 211 of Proceedings of Machine Learning Research, PMLR, 2023, pp. 665--678.

\bibitem{Hugemann2002DriverRT}
W.~Hugemann, \href{https://api.semanticscholar.org/CorpusID:58933066}{Driver reaction times in road traffic}, 2002.
\newline\urlprefix\url{https://api.semanticscholar.org/CorpusID:58933066}

\bibitem{pmlr-v168-cosner22a}
R.~Cosner, M.~Tucker, A.~Taylor, K.~Li, T.~Molnar, W.~Ubelacker, A.~Alan, G.~Orosz, Y.~Yue, A.~Ames, Safety-aware preference-based learning for safety-critical control, in: R.~Firoozi, N.~Mehr, E.~Yel, R.~Antonova, J.~Bohg, M.~Schwager, M.~Kochenderfer (Eds.), Proceedings of The 4th Annual Learning for Dynamics and Control Conference, Vol. 168 of Proceedings of Machine Learning Research, PMLR, 2022, pp. 1020--1033.

\bibitem{tucker2022polar}
M.~Tucker, K.~Li, Y.~Yue, A.~D. Ames, Polar: Preference optimization and learning algorithms for robotics (2022).
\newblock \href {http://arxiv.org/abs/2208.04404} {\path{arXiv:2208.04404}}.

\end{thebibliography}






\appendix 
\section{Discrepancy-Aware Cost Map Construction Details} \label{Appendix:1}
This section details the process of constructing a discrepancy-aware cost map given an occupancy map $\mathcal{O}$. For every occupied grid, its neighboring grids within a distance of $ (r_{\Delta T} + r_{ego})$, referred to as the safety buffer, will also be considered as occupied grids to ensure robust obstacle avoidance given model inaccuracy and vehicle geometry.

An enlarged grid map $\hat{\mathcal{O}} \in \mathbb{R}^{l_{map} + 2N_{\epsilon} \times w_{map} + 2N_{\epsilon}}$ is initialized, and the nominal map is shifted and aggregated to form the collision buffer grid map \eqref{eq:collision_buffer_grid_map} and visualized in Figure \ref{fig:first_page_summary}. Specifically, We shift the nominal occupancy map by $N_{\epsilon}$ grids in both column-wise and row-wise positive and negative directions, as described in Algorithm \ref{alg:2}. In total, there are $4N_{\epsilon} + 2$ shifted grid maps which are grid-wise aggregated to obtain the grid map with collision buffer,
\begin{equation} \label{eq:collision_buffer_grid_map}
    \hat{\mathcal{O}}_{buffer} \triangleq \frac{1}{4N_{\epsilon} + 2}\sum_{i = -N_{\epsilon}}^{N_{\epsilon}} \sum_{j = -N_{\epsilon}}^{N_{\epsilon}} \frac{\hat{\mathcal{O}}^{s}_{i,j}}{100},
\end{equation}
where we use $\hat{\mathcal{O}}^{s}_{i,j} \triangleq \hat{\mathcal{O}}(N_{\epsilon} + i:N_{\epsilon} + l_{map} + i - 1, N_{\epsilon} + j:N_{\epsilon} + w_{map} + j - 1)$ as the grid shifting operation. The normalization factor $100$ is applied on the enlarged occupancy maps to obtain a value within $[0,1]$.

The discrepancy-aware cost map $\mathcal{C}_{\epsilon}$ is constructed by applying a tunable cost multiplier $\alpha_{shift}/N_{shift,\epsilon}$ to this enlarged occupancy map $\hat{\mathcal{O}}_{buffer}$ The multiplier $\alpha_{shift}/N_{shift,\epsilon}$ to the $N_{\epsilon}$-shifted occupancy maps is a design parameter where $\alpha_{shift} \in \mathbb{R}_{>0}$ is a cost multiplier similar to $\alpha_{ISS}$ in \eqref{eq:augmented_cost} and the positive integer $N_{shift} \in \{1,\hdots, N_{\epsilon}\}$ serves as a cost decay. This cost adjustment accounts for both conservative obstacle avoidance and model mismatches, ensuring that the true and uncertain system flow remains collision-free given the risk tolerance. The final cost map balances obstacle collision penalties and tracking costs (Algorithm \ref{alg:2}).  The design parameter $\alpha_{shift}/N_{\epsilon}$ needs to be lower bounded by the maximum tracking cost $\overline{L}_{track}$ to penalize unsafe trajectories. A larger $\alpha_{shift}$ leads to more conservative obstacle avoidance behavior, and a smaller $\alpha_{shift}$ reduces the cost associated with the avoidance of the safe buffer due to model mismatches. The discrepancy-aware cost map, denoted as $\mathcal{C}_{\epsilon}$ (See Algorithm \ref{alg:2}), is the centering $l_{map}$ and $w_{map}$ grids of the following enlarged cost map:
\begin{equation}
        \mathcal{C}_{buffer} \triangleq \sum_{i = -N_{\epsilon}}^{N_{\epsilon}} \sum_{j = -N_{\epsilon}}^{N_{\epsilon}} \alpha_{shift}/\underbrace{\sqrt{i^2 + j^2}}_{\triangleq N_{shift,\epsilon}}\hat{\mathcal{O}}^{s}_{i,j}.
\end{equation}
Lastly, we apply a maximum cost map threshold $\mathcal{C}_{\epsilon} = \min(\mathcal{C}_\epsilon,\overline{L}_{track})$ with the maximum tracking cost grid-wise to even the cost penalty for obstacle collision.

\end{document}